\documentclass{article}

\PassOptionsToPackage{numbers,sort&compress}{natbib}

\usepackage[preprint]{neurips_2026}

\usepackage[utf8]{inputenc}
\usepackage[T1]{fontenc}
\usepackage{url}
\usepackage{amsmath,amsthm,amssymb,amsfonts}
\usepackage{enumerate}
\usepackage{subfigure}
\usepackage{multirow}
\usepackage{hhline}
\usepackage{xcolor}
\usepackage{hyperref}
\usepackage{graphicx}
\graphicspath{{figure/}}
\ifx\pdfoutput\undefined
  \DeclareGraphicsExtensions{.eps}
\else
  \ifnum\pdfoutput>0
    \DeclareGraphicsExtensions{.png,.pdf,.jpg}
  \else
    \DeclareGraphicsExtensions{.eps}
  \fi
\fi
\usepackage{graphics}
\usepackage{algorithm}
\usepackage{algpseudocode}
\usepackage{epsfig,psfrag}
\usepackage{tabularx}
\usepackage{tabulary}
\usepackage{booktabs}
\usepackage{nicefrac}
\usepackage{microtype}
\usepackage[noabbrev,capitalise,nameinlink]{cleveref}


\usepackage{kotex}

\hypersetup{
  breaklinks=true,
  hypertexnames=false,
  colorlinks=true,
  linkcolor=blue,
  citecolor=blue,
  urlcolor=blue
}

\DeclareMathOperator{\argmax}{arg\,max}
\DeclareMathOperator{\argmin}{arg\,min}
\DeclareMathOperator{\co}{co}
\DeclareMathOperator{\diag}{diag}

\newcommand{\R}{\mathbb{R}}
\newcommand{\E}{\mathbb{E}}
\newcommand{\Prob}{\mathbb{P}}
\newcommand{\calF}{\mathcal{F}}
\newcommand{\calS}{\mathcal{S}}
\newcommand{\calA}{\mathcal{A}}
\newcommand{\calM}{\mathcal{M}}
\newcommand{\calH}{\mathcal{H}}
\newcommand{\jsr}{\rho}

\newtheorem{theorem}{Theorem}

\newtheorem{assumption}{Assumption}

\newtheorem{lemma}{Lemma}
\newtheorem*{lemma*}{Lemma}
\newtheorem{definition}{Definition}
\newtheorem{proposition}{Proposition}

\crefname{assumption}{Assumption}{Assumptions}
\Crefname{assumption}{Assumption}{Assumptions}

\allowdisplaybreaks

\title{Lyapunov-Certified Direct Switching Theory for Q-Learning}

\author{%
Donghwan Lee\\
Department of Electrical Engineering\\
Korea Advanced Institute of Science and Technology (KAIST)\\
Daejeon 34141, South Korea\\
\texttt{donghwan@kaist.ac.kr}
}

\begin{document}

\maketitle

\begin{abstract}
Q-learning is a fundamental algorithmic primitive in reinforcement learning. This paper develops a new framework for analyzing Q-learning from a switching linear system (SLS) viewpoint. In particular, we derive a stochastic SLS representation of the Q-learning error, and a finite-time error analysis through the joint spectral radius (JSR) of the corresponding SLS model, where the JSR is the exact worst-case exponential rate of the associated SLS. To the best of our knowledge, this is the first convergence rate analysis of standard Q-learning whose leading exponential rate is expressed through the JSR. The resulting rate is tied to the intrinsic worst-case exponential rate of the direct SLS representation and can be sharper than row-sum upper bounds when those bounds are conservative.
\end{abstract}

\section{Introduction}
Q-learning~\cite{watkins1992q} is a foundational algorithm in reinforcement
learning (RL)~\cite{sutton1998reinforcement} for solving discounted Markov
decision processes (MDPs) with unknown transition kernels. Its convergence has
been studied extensively over the past several decades. Classical analyses
primarily establish asymptotic convergence~\cite{tsitsiklis1994asynchronous,jaakkola1994convergence,
borkar2000ode,lee2020unified}. While these results are fundamental,
asymptotic convergence alone does not quantify how rapidly the iterates
approach the solution, and therefore does not fully characterize the efficiency
of the algorithm. This limitation has motivated a growing body of finite-time
convergence analyses, which provide explicit bounds on the progress of the
iterates toward the optimal Q-function. Recent advances in finite-time analysis
include~\cite{szepesvari1998asymptotic,kearns1999finite,even2003learning,
beck2012error,wainwright2019stochastic,qu2020finite,li2020sample,
chen2021lyapunov}. Most existing results view Q-learning as a
nonlinear stochastic approximation scheme~\cite{kushner2003stochastic} and
rely on the contraction property of the Bellman optimality operator.

A useful alternative viewpoint is to regard Q-learning as a discrete-time
stochastic switching linear system (SLS)~\cite{lin2009stability,liberzon2003switching}.
This perspective was developed in~\cite{lee2020unified,lee2023discrete,
lee2024final,lim2024diminishing} and was used to prove a final-iterate bound for constant-step-size
Q-learning. In that formulation, the error dynamics are affine rather than
linear, because the greedy policy selected by the current iterate may differ
from an optimal policy. The resulting affine term is controlled through upper
and lower comparison systems. More specifically, the upper comparison system
remains an SLS, whereas the lower comparison system is linear.
Although this comparison-system approach is useful for deriving valid
finite-time bounds, it has an inherent limitation: it controls the Q-learning
error through auxiliary systems rather than directly exploiting the SLS
structure of the original error recursion. Consequently, intrinsic quantities
from SLS theory, such as the joint spectral radius
(JSR)~\cite{rota1960note,tsitsiklis1997lyapunov,jungers2009joint}, are
difficult to apply directly to the original Q-learning dynamics.

To overcome this limitation, this paper develops a direct SLS
representation of the Q-learning error. The key observation is that the Bellman
maximization error can be represented exactly as an average of action-wise
Q-errors under a suitably chosen stochastic policy. Substituting this identity
into the Q-learning recursion yields an SLS conditional-mean error
system, while the stochastic component remains a martingale difference. The
corresponding convergence rate is characterized by the JSR of the
resulting SLS. This is the central distinction from existing
Bellman-contraction and comparison system analyses: the rate is attached to the
exact direct SLS drift rather than to an auxiliary contraction or row-sum
upper bound. To the best of our knowledge, this is the first convergence rate
framework for standard Q-learning whose leading exponential rate is formulated
in terms of such a direct JSR. Since the JSR is the exact worst-case exponential
rate of the SLS, the resulting transient bound is tied to the intrinsic
drift rate of this direct representation and can be sharper than row-sum
comparison bounds when those bounds are conservative.

We first develop the framework under an i.i.d. observation model with a
constant step-size, where the direct SLS structure is most transparent.
We then describe a Markovian observation extension based on the same
SLS representation. In that
extension, the pathwise decomposition remains unchanged, while the i.i.d.
covariance recursion is replaced by a conservative Poisson-equation estimate
for the Markovian noise. To certify the deterministic transient, we use Lyapunov certificates for the corresponding SLS model.

Finally, we emphasize that the goal of this paper is to offer a new perspective
and a methodological framework for understanding the convergence of Q-learning.
The proposed framework is not intended to replace existing analyses based on
Bellman contractions or stochastic approximation, nor does it claim uniformly
improved finite-time error bounds across all problem instances.
\section{Preliminaries}\label{sec:preliminaries}
\subsection{Notation}
The set of real numbers is denoted by $\R$; $\R^m$ is the $m$-dimensional Euclidean space; and $\R^{m\times r}$ is the set of all $m\times r$ real matrices. For a matrix $A$, $A^\top$ denotes its transpose. For symmetric matrices, $A\succ0$ and $A\succeq0$ denote positive definiteness and positive semidefiniteness, and $A\preceq B$ means $B-A\succeq0$. The identity matrix is denoted by $I$. For vectors, $e_i$ is the $i$th standard basis vector, with dimension clear from context, and $\otimes$ denotes the Kronecker product. The quantities $\lambda_{\min}(A)$ and $\lambda_{\max}(A)$ denote the minimum and maximum eigenvalues of a symmetric matrix $A$, respectively. For a finite set $\calS$, $|\calS|$ denotes its cardinality. We write $\Delta_m:=
\left\{q\in\R^m:q_i\ge0,\ \sum_{i=1}^{m}q_i=1\right\}$ for the probability simplex in $\R^m$. For a finite matrix family $\calH=\{A_1,\ldots,A_N\}$, the notation $\co(\calH)$ denotes the convex hull $\co(\calH)
:=
\left\{
\sum_{i=1}^N \lambda_i A_i:
\lambda_i\ge0,\ \sum_{i=1}^N\lambda_i=1
\right\}$. For probability distributions $p$ and $q$ on the same finite set $\mathcal X$,
$\|p-q\|_{\mathrm{TV}}:=\frac12\sum_{x\in\mathcal X} |p(x)-q(x)|$.

\subsection{Switching linear system (SLS)}
Let us consider the discrete-time stochastic SLS with a possible affine term~\cite{liberzon2003switching,lin2009stability,shorten2007stability}
\[
  x_{k+1}=A_{\sigma_k}x_k+b_{\sigma_k}+\xi_k,
\]
where $\sigma_k \in \{1,2,\ldots,M\}$ is a switching signal, $A_{\sigma_k}$ is selected from a prescribed family of matrices ${\cal H}:=\{A_1,A_2,\ldots,A_M\}$, which is called a switching family, $b_{\sigma_k}$ is a mode-dependent affine term, and $\xi_k$ is a disturbance or martingale noise. When $b_{\sigma_k}=0$, the deterministic part reduces to an SLS,
\[
  x_{k+1}=A_{\sigma_k}x_k.
\]
The worst-case exponential rate of the SLS family is characterized by the joint spectral radius (JSR)~\cite{tsitsiklis1997lyapunov,rota1960note,blondel2005computational,jungers2009joint}, defined as follows.
\begin{definition}\label{def:jsr}
For a bounded set of matrices $\calH\subset\R^{m\times m}$, its joint spectral radius is
\[
\jsr(\calH)
:=
\lim_{k\to\infty}
\sup_{A_1,\ldots,A_k\in\calH}
\|A_k\cdots A_1\|^{1/k},
\]
where the value is independent of the chosen submultiplicative norm~\cite{rota1960note,jungers2009joint}. When $\calH$ is finite, the supremum for each fixed product length is a maximum over products generated by matrices in $\calH$.  For a finite family $\calH$, the notation $\jsr(\co(\calH))$ means the JSR computed when each factor in a product is allowed to be any convex combination of matrices in $\calH$.
\end{definition}

\begin{lemma}\label{lem:jsr-stability}
Consider the SLS
\[
  x_{k+1}=A_{\sigma_k}x_k,
  \qquad
  A_{\sigma_k}\in\calH,
\]
where \(\calH\subset\R^{m\times m}\) is finite.  If \(\jsr(\calH)<1\), then \(x_k\to0\) for every initial condition and every switching sequence.  Moreover, the convergence is exponential.
\end{lemma}
\begin{proof}
Choose \(\beta\) such that \(\jsr(\calH)<\beta<1\).  By the definition of the JSR, there exists an integer \(K\ge1\) such that
\[
  \sup_{A_1,\ldots,A_k\in\calH}
  \|A_k\cdots A_1\|^{1/k}\le \beta,
  \qquad k\ge K.
\]
For the finitely many remaining product lengths, enlarge the multiplicative constant so that
\[
  \sup_{A_1,\ldots,A_k\in\calH}
  \|A_k\cdots A_1\|
  \le C\beta^k,
  \qquad k\ge0,
\]
with the product of length zero interpreted as the identity.  Hence, along any switching sequence,
\[
  \|x_k\|
  \le C\beta^k\|x_0\|,
  \qquad k\ge0.
\]
Since \(\beta<1\), this proves exponential convergence to zero.
\end{proof}
For Q-learning~\cite{sutton1998reinforcement}, this perspective is natural because the greedy policy induced by the current iterate plays the role of the switching signal.
In the final-iterate analysis of~\cite{lee2024final}, constant step-size Q-learning under an independent and identically distributed (i.i.d.) observation model is represented by a stochastic affine SLS recursion. The affine term appears because the greedy policy associated with $Q_k$ need not be optimal. We keep the SLS viewpoint, but replace the affine representation with a direct SLS representation obtained from a stochastic policy linearization of the Bellman maximization error.

\subsection{Discounted Markov decision processes and Bellman operator}
Let us consider a finite discounted Markov decision process (MDP)~\cite{puterman2014markov} with state space $\calS=\{1,\ldots,|\calS|\}$, action space $\calA=\{1,\ldots,|\calA|\}$, transition probability $P(s'\mid s,a)$ for $(s,a,s')\in\calS\times\calA\times\calS$, one-step reward $r(s,a,s')$ for $(s,a,s')\in\calS\times\calA\times\calS$, expected reward $R(s,a):=\sum_{s'\in\calS}P(s'\mid s,a)r(s,a,s'), (s,a)\in\calS\times\calA$, and discount factor $\gamma\in(0,1)$. Let $\Theta$ denote the set of deterministic stationary policies $\pi:\calS\to\calA$. A Q-function is viewed as a vector $Q\in\R^{|\calS|\,|\calA|}$ using a fixed ordering of state-action pairs. We use the compact notation
\[
P:=
\begin{bmatrix}
P_1\\ \vdots\\ P_{|\calA|}
\end{bmatrix}
\in\R^{(|\calS|\,|\calA|)\times |\calS|},
\qquad
R:=
\begin{bmatrix}
R(\cdot,1)\\ \vdots\\ R(\cdot,|\calA|)
\end{bmatrix}
\in\R^{|\calS|\,|\calA|},
\]
where $P_a=P(\cdot\mid\cdot,a)\in\R^{|\calS|\times |\calS|}, a\in\calA$.
For any stochastic policy $\mu:\calS\to\Delta_{|\calA|}$, define
\[
\Pi^\mu
:=
\begin{bmatrix}
\mu(1)^\top\otimes e_1^\top\\
\mu(2)^\top\otimes e_2^\top\\
\vdots\\
\mu(|\calS|)^\top\otimes e_{|\calS|}^\top
\end{bmatrix}
\in\R^{|\calS|\times(|\calS|\,|\calA|)}.
\]
Then $P\Pi^\mu\in\R^{(|\calS|\,|\calA|)\times(|\calS|\,|\calA|)}$ is the state-action transition matrix under $\mu$. For a deterministic policy $\pi\in\Theta$, we use the same notation $\Pi^\pi$ by identifying $\pi(s)$ with its one-hot encoding for each $s\in\calS$. For $Q\in\R^{|\calS|\,|\calA|}$, let us define
\[
  V_Q(s):=\max_{a\in\calA}Q(s,a),
  \qquad s\in\calS,
\]
and $V_Q:=(V_Q(1),\ldots,V_Q(|\calS|))^\top \in \R^{|\calS|}$.
The Bellman optimality operator is $F(Q):=R+\gamma PV_Q$. The optimal Q-function $Q^*$ is the unique fixed point of $F$, i.e., $Q^*=F(Q^*)$, and $V^*:=V_{Q^*}$.
\begin{definition}\label{def:Rmax}
Since $\calS$ and $\calA$ are finite, let us define
\[
R_{\max}:=
\max_{(s,a,s')\in\calS\times\calA\times\calS}|r(s,a,s')|.
\]
\end{definition}
It is well-known that the optimal Q-function then satisfies $\|Q^*\|_\infty\le \frac{R_{\max}}{1-\gamma}$.

\subsection{Q-learning setting}\label{subsec:iid-qlearning}
We consider the standard asynchronous Q-learning recursion~\cite{sutton1998reinforcement} with a constant step-size $\alpha$ under an i.i.d.\ observation model. Throughout the paper, we assume that at each time $k$, a state-action pair $(s_k,a_k)$ is sampled independently with $\Prob(s_k=s,a_k=a)=d(s,a):=p(s)b(a\mid s), (s,a)\in\calS\times\calA$, where $p$ is a state sampling distribution and $b$ is a fixed behavior policy. Then, $s'_k\sim P(\cdot\mid s_k,a_k)$ and $r_{k+1}:=r(s_k,a_k,s'_k)$ are sampled. The update is
\[
Q_{k+1}(s_k,a_k)
=
Q_k(s_k,a_k)+\alpha
\left(
r_{k+1}+\gamma\max_{u\in\calA}Q_k(s'_k,u)-Q_k(s_k,a_k)
\right),
\]
while all coordinates $(s,a)\in\calS\times\calA$ with $(s,a)\ne(s_k,a_k)$ remain unchanged. Let $\{\calF_k\}_{k\ge0}$ be the natural filtration of the Q-learning process
\[
  \calF_0:=\sigma(Q_0),
  \qquad
  \calF_k:=
  \sigma\!\left(Q_0,\{(s_t,a_t,s'_t,r_{t+1}):0\le t\le k-1\}\right),
  \quad k\ge1.
\]
Then, $Q_k$ is $\calF_k$-measurable. Let $D:=\diag(d(s,a))_{(s,a)\in\calS\times\calA}\in\R^{(|\calS|\,|\calA|)\times(|\calS|\,|\calA|)}$.
For the sampled coordinate, define the random vector
\[
\zeta_k:=(e_{a_k}\otimes e_{s_k})
\left(
r_{k+1}+\gamma\max_{u\in\calA}Q_k(s'_k,u)-Q_k(s_k,a_k)
\right).
\]
Then, we have $\E[\zeta_k\mid\calF_k]
=
D(F(Q_k)-Q_k)$. Accordingly, the martingale-difference noise is defined by
\begin{equation}\label{eq:wk-def}
w_k
:=
(e_{a_k}\otimes e_{s_k})
\left(
r_{k+1}
+\gamma\max_{u\in\calA}Q_k(s'_k,u)
-Q_k(s_k,a_k)
\right)
-
D(F(Q_k)-Q_k).
\end{equation}
Therefore, the vector form of Q-learning is
\begin{align*}
Q_{k+1}
=
Q_k+
\alpha\{D(F(Q_k)-Q_k)+w_k\},
\end{align*}
where $w_k$ is defined in~\cref{eq:wk-def} and satisfies $\E[w_k\mid\calF_k]=0$ with respect to the natural filtration $\{\calF_k\}_{k\ge0}$ defined above.

\begin{assumption}\label{assump:basic}
The following conditions hold throughout the paper.
\begin{enumerate}[(i)]
\item $d(s,a)>0$ for every $(s,a)\in\calS\times\calA$.
\item The step size satisfies $\alpha\in(0,1)$.
\item The initial Q-table $Q_0\in\R^{|\calS|\,|\calA|}$ is deterministic.
\end{enumerate}
\end{assumption}
For convenience, we define the notations
\[
  d_{\min}:=\min_{(s,a)\in\calS\times\calA}d(s,a),
  \qquad
  d_{\max}:=\max_{(s,a)\in\calS\times\calA}d(s,a).
\]
Since \(d\) is a probability distribution on
\(|\calS|\,|\calA|\) coordinates, we have
\[
  0<d_{\min}\le \frac{1}{|\calS|\,|\calA|}.
\]

\section{SLS representation of Q-learning}\label{sec:direct-switching}

We now remove the affine offset from the error recursion. The only ingredient needed for this step is the state-wise convexity fact in~\cite{goyal2023first} stated and proved in Appendix~\ref{app:proof-stochastic-policy-linearization}, \cref{lem:stochastic-policy-linearization}: the Bellman maximization error can be written exactly as an average of action-wise Q-errors under a suitable stochastic policy. Substituting this linearization into the Q-learning recursion removes the affine term exactly.
\begin{theorem}\label{thm:direct-representation}
Under \cref{assump:basic}, let $\{\calF_k\}_{k\ge0}$ be the natural filtration defined in \cref{subsec:iid-qlearning}.  Then, for each $k\ge0$, there exists an $\calF_k$-measurable stochastic policy $\mu_k$ such that
\begin{equation*}
  (Q_{k+1}-Q^*)=M_{\mu_k}(Q_k-Q^*)+\alpha w_k,
\end{equation*}
where for any stochastic policy $\mu$,
\[
  M_\mu:=I-\alpha D+\alpha\gamma DP\Pi^\mu,
\]
and
\[
w_k
:=
(e_{a_k}\otimes e_{s_k})
\left(
r_{k+1}+\gamma\max_{u\in\calA}Q_k(s'_k,u)-Q_k(s_k,a_k)
\right)
-
D(F(Q_k)-Q_k)
\]
and $\E[w_k\mid\calF_k]=0$. Consequently,
\begin{equation}\label{eq:conditional-mean-switching}
  \E[Q_{k+1}-Q^*\mid\calF_k]
  =M_{\mu_k}(Q_k-Q^*).
\end{equation}
\end{theorem}
\noindent The proof is given in Appendix~\ref{app:proof-direct-representation}.
The corresponding switching family is denoted by
\[
\calM_\alpha
:=
\left\{
M_\pi:=I-\alpha D+\alpha\gamma DP\Pi^\pi:
\pi\in\Theta
\right\}.
\]
We note that the policies in~\cref{thm:direct-representation} are stochastic in $\co(\calM_\alpha)$, whereas we will consider the JSR of the above switching family $\calM_\alpha$, which is defined through deterministic modes. The next lemma connects the two descriptions.
\begin{lemma}\label{lem:convex-hull-M}
For every stochastic policy $\mu$, $M_\mu\in\co(\calM_\alpha)$.  Moreover, when the JSR of $\co(\calM_\alpha)$ is interpreted as the JSR over the compact set of all convex combinations of matrices in $\calM_\alpha$, we have
\[
  \jsr(\co(\calM_\alpha))=\jsr(\calM_\alpha).
\]
\end{lemma}
\noindent The proof is given in Appendix~\ref{app:proof-convex-hull-M}.
Consequently, it is enough to verify \(\jsr(\calM_\alpha)<1\) when proving convergence of the SLS; no separate condition on \(\jsr(\co(\calM_\alpha))\) is needed.

We next compare this direct JSR with the familiar row-sum rate. The row-sum rate introduced in~\cite{lee2023discrete,lee2024final} is
\[
  \rho_{\mathrm{row}}:=1-\alpha d_{\min}(1-\gamma).
\]
Under \cref{assump:basic}, \(\rho_{\mathrm{row}}\in(0,1)\). The next lemma connects this row-sum rate with the direct switching family above.
\begin{lemma}\label{lem:row-sum-jsr-upper-bound}
Under \cref{assump:basic}, the row-sum rate is an upper bound on the direct JSR:
\[
  \jsr(\calM_\alpha)\le\rho_{\mathrm{row}}.
\]
More generally, for every stochastic policy \(\mu\),
\[
  \|M_\mu\|_\infty\le\rho_{\mathrm{row}}.
\]
\end{lemma}
\noindent The proof is given in Appendix~\ref{app:proof-row-sum-jsr-upper-bound}.
We now formalize the deterministic SLS representation associated with \cref{thm:direct-representation}:
\begin{equation}\label{eq:direct-discrete-switched-system}
  x_{k+1}=M_{\mu_k}x_k,
  \qquad
  M_{\mu_k}\in\co(\calM_\alpha),
  \qquad k\ge0,
\end{equation}
where the switching sequence may be arbitrary and may select any stochastic-policy matrix generated by \cref{thm:direct-representation}. Although this is the deterministic counterpart of the conditional-mean recursion, its stability identifies the drift mechanism that will later be combined with martingale noise estimates for stochastic Q-learning.

The following lemma verifies the JSR stability condition in \cref{lem:jsr-stability} for the direct Q-learning family.
\begin{lemma}\label{lem:direct-jsr-strict}
Under \cref{assump:basic}, the direct switching family is exponentially stable in the JSR sense:
\[
  \jsr(\calM_\alpha)<1.
\]
In fact,
\[
  \jsr(\calM_\alpha)\le \rho_{\mathrm{row}}=1-\alpha d_{\min}(1-\gamma)<1.
\]
\end{lemma}
\begin{proof}
This follows immediately from \cref{lem:row-sum-jsr-upper-bound} and the inequalities \(\alpha\in(0,1)\), \(d_{\min}>0\), and \(\gamma\in(0,1)\).
\end{proof}
Therefore, by \cref{lem:direct-jsr-strict}, the deterministic SLS in \cref{eq:direct-discrete-switched-system} satisfies the JSR stability condition.

\begin{proposition}\label{prop:direct-rate}
Consider the discrete SLS in \cref{eq:direct-discrete-switched-system}.  For every \(\varepsilon>0\) such that
\[
  \beta_\varepsilon:=\jsr(\calM_\alpha)+\varepsilon<1,
\]
there exists a norm \(p_\varepsilon\) on \(\R^{|\calS|\,|\calA|}\) such that
\[
  p_\varepsilon(M_\mu x)
  \le
  \beta_\varepsilon p_\varepsilon(x),
  \qquad
  \forall x\in\R^{|\calS|\,|\calA|},
\]
for every stochastic policy \(\mu\).  Hence \(p_\varepsilon\) is a common Lyapunov norm for \cref{eq:direct-discrete-switched-system}; every trajectory satisfies
\[
  p_\varepsilon(x_k)\le \beta_\varepsilon^k p_\varepsilon(x_0),
  \qquad k\ge0,
\]
and therefore \(x_k\to0\) exponentially for every switching sequence.
\end{proposition}
\noindent The proof is given in Appendix~\ref{app:proof-direct-rate}.
Consequently, the deterministic SLS in \cref{eq:direct-discrete-switched-system} converges to zero for every switching sequence.  The finite-time stochastic analysis below follows the same drift-certification idea, with martingale noise handled separately.  By \cref{lem:row-sum-jsr-upper-bound}, the row-sum rate is only an upper bound on the direct switching family \(\calM_\alpha\) and can be strictly larger than the JSR of \(\calM_\alpha\).

\section{Finite-time error bound under i.i.d. observation}
\label{sec:jsr-lyapunov-qlearning}
The SLS representation in \cref{thm:direct-representation} makes it possible to analyze the finite-time behavior of Q-learning through SLS theory and the JSR~\cite{liberzon2003switching,lin2009stability,shorten2007stability,rota1960note,tsitsiklis1997lyapunov,jungers2009joint}. To state the bound, we first construct a Lyapunov function used to certify the stability. Fix \(\varepsilon>0\) and write
\[
  \beta_\varepsilon:=\jsr(\calM_\alpha)+\varepsilon.
\]
For each integer \(t\ge0\), let us define
\[
V_\varepsilon^t(x)
:=
\sum_{\ell=0}^{t}
\beta_\varepsilon^{-2\ell}
\max_{\pi_1,\ldots,\pi_\ell\in\Theta}
\left\|
M_{\pi_\ell}\cdots M_{\pi_1}x
\right\|_2^2,
\qquad x\in\R^{|\calS|\,|\calA|},
\]
with the empty-product term equal to \(\|x\|_2^2\), and define
\[
V_\varepsilon^\infty(x):=\lim_{t\to\infty}V_\varepsilon^t(x).
\]
This Lyapunov function construction is detailed in Appendix~\ref{app:jsr-lyapunov-aux}, \cref{eq:Veps-t-direct,eq:Veps-infty-direct}.  When \(\beta_\varepsilon<1\), Appendix~\ref{app:jsr-lyapunov-aux}, \cref{lem:Veps-existence} shows that \(V_\varepsilon^\infty\) is finite and satisfies the norm-equivalence estimate
\[
  \|x\|_2^2
  \le
  V_\varepsilon^\infty(x)
  \le
  C_\varepsilon\|x\|_2^2,
  \qquad
  \forall x\in\R^{|\calS|\,|\calA|},
\]
for some \(C_\varepsilon\ge1\). The theorem below uses this Lyapunov function to control the SLS conditional-mean drift and then converts the Lyapunov estimate into a finite-time \(\ell_\infty\)-error bound.
\begin{theorem}
\label{thm:finite-time-Veps}
Let \cref{assump:basic} hold. Fix any \(\varepsilon>0\) such that \(\beta_\varepsilon<1\), and let \(V_\varepsilon^\infty\) and \(C_\varepsilon\) be defined by the preceding Lyapunov construction. Let \(W_{\max}\) be the noise bound defined in Appendix~\ref{app:proof-bounded-noise}, \cref{eq:Wmax-general}. Then, for all \(k\ge0\),
\begin{equation}\label{eq:Veps-moment-bound}
\E[V_\varepsilon^\infty(Q_k-Q^*)]
\le
\beta_\varepsilon^{2k}V_\varepsilon^\infty(Q_0-Q^*)
+
\frac{\alpha^2C_\varepsilon^2W_{\max}}{1-\beta_\varepsilon^2}
\left(1-\beta_\varepsilon^{2k}\right).
\end{equation}
Consequently,
\begin{equation}\label{eq:Veps-final-bound-general}
\E[\|Q_k-Q^*\|_\infty]
\le
\sqrt{C_\varepsilon}\,
\beta_\varepsilon^k
\|Q_0-Q^*\|_2
+
\alpha C_\varepsilon
\sqrt{
\frac{W_{\max}}{1-\beta_\varepsilon^2}
}.
\end{equation}
\end{theorem}
The proof is given in Appendix~\ref{app:proof-finite-time-Veps}.
The residual term in~\cref{eq:Veps-final-bound-general} should be interpreted carefully. The theorem does not by itself guarantee that the displayed noise-floor term $\alpha C_\varepsilon
  \sqrt{
  \frac{W_{\max}}{1-\beta_\varepsilon^2}
  }$ vanishes as \(\alpha\downarrow0\), because the norm-equivalence constant \(C_\varepsilon\) in the displayed JSR Lyapunov certificate may also depend on \(\alpha\) and \(\varepsilon\). Therefore, the bound is different from standard finite-time error analyses in which the constant-step-size noise floor is designed to be made arbitrarily small by decreasing the step-size. The advantage of the present interpretation is not a tighter noise floor, but a sharper transient rate: the exponential term is governed by the JSR rate \(\beta_\varepsilon\).
For instance, the finite-time error analysis in~\cite{lee2024final} gives the row-sum rate \(\rho_{\mathrm{row}}=1-\alpha d_{\min}(1-\gamma)\). The JSR-based bound in~\cref{eq:Veps-final-bound-general} instead uses \(\beta_\varepsilon = \jsr(\calM_\alpha)+\varepsilon\). Since \(\jsr(\calM_\alpha)\le\rho_{\mathrm{row}}\), the exponential transient in \cref{eq:Veps-final-bound-general} can be strictly faster than the comparison-system transient when the row-sum bound is loose. The numerical illustration in \cref{sec:numerical-jsr-illustration} exhibits this gap explicitly.

We can also derive an alternative bound whose noise floor is explicitly controlled by the step-size $\alpha$. The proof compares the SLS dynamics with a stochastic reference system driven by the same noise so that the noise cancels pathwise in the residual recursion. Concretely, for a fixed stochastic policy \(\bar\mu\), let us consider the reference system
\[
  x_{k+1}=M_{\bar\mu}x_k+\alpha w_k,
  \qquad
  x_0=Q_0-Q^*.
\]
Subtracting this recursion from the direct error recursion gives
\[
  (Q_{k+1}-Q^*)-x_{k+1}
  =
  M_{\mu_k}\big((Q_k-Q^*)-x_k\big)
  +(M_{\mu_k}-M_{\bar\mu})x_k,
\]
so the martingale noise \(\alpha w_k\) cancels exactly in the residual. The remaining terms are controlled by the JSR product bound and a covariance estimate for the reference system. More details are given in Appendix~\ref{app:direct-reference-filter-noise-floor}.
\begin{theorem}
\label{thm:direct-reference-filter-noise-floor}
Let \cref{assump:basic} hold. Let \(W_{\max}\) be the deterministic noise bound defined in Appendix~\ref{app:proof-bounded-noise}, \cref{eq:Wmax-general}.
Fix any \(\varepsilon>0\) such that
\[
  \beta_\varepsilon:=\jsr(\calM_\alpha)+\varepsilon<1,
\]
and let \(K_{\beta_\varepsilon}\) be defined in Appendix~\ref{app:direct-reference-filter-noise-floor}, \cref{lem:product-growth-constant-direct}. Then, for all \(k\ge0\),
\begin{equation}
\label{eq:direct-reference-filter-bound}
\begin{aligned}
\E[\|Q_k-Q^*\|_\infty]
\le\;&
|\calS|\,|\calA|\sqrt{\frac{\alpha W_{\max}}{d_{\min}(1-\gamma)}}
\left(
1+
\frac{2\gamma d_{\max}}{d_{\min}(1-\gamma)}
\right) \\
&+
K_{\beta_\varepsilon}\beta_\varepsilon^k\|Q_0-Q^*\|_\infty \\
&+
2\alpha\gamma d_{\max}K_{\beta_\varepsilon}^2
k\beta_\varepsilon^{k-1}\|Q_0-Q^*\|_\infty,
\end{aligned}
\end{equation}
with the convention that \(k\beta_\varepsilon^{k-1}=0\) when \(k=0\).
\end{theorem}
The proof is given in Appendix~\ref{app:direct-reference-filter-noise-floor}.
The first term in \cref{eq:direct-reference-filter-bound} is the noise floor, while the remaining two terms are transients governed by the JSR rate. For fixed problem data and initialization, this noise-floor term is of order \(\sqrt{\alpha}\) and therefore vanishes as \(\alpha\downarrow0\).

\section{Finite-time error bound under Markovian observation}
This section applies the reference-filter idea from \cref{thm:direct-reference-filter-noise-floor} to Markovian observations and derives a finite-time error bound for single-trajectory Q-learning. Let \(b(a\mid s)\) generate
\[
  s_{k+1}\sim P(\cdot\mid s_k,a_k),
  \qquad
  r_{k+1}=r(s_k,a_k,s_{k+1}),
  \qquad
  a_{k+1}\sim b(\cdot\mid s_{k+1}),
\]
and put \(X_k:=(s_k,a_k)\). The behavior-induced chain has transition kernel
\[
  P^b(s',a'\mid s,a):=P(s'\mid s,a)b(a'\mid s').
\]
The single-trajectory update is \begin{equation}\label{eq:markovian-single-trajectory-update-ref}
Q_{k+1}(s_k,a_k)
=
Q_k(s_k,a_k)
+
\alpha
\left(
 r_{k+1}
 +
 \gamma\max_{u\in\calA}Q_k(s_{k+1},u)
 -
 Q_k(s_k,a_k)
\right),
\end{equation}
while all other coordinates remain unchanged. Unlike the i.i.d. observation model, the sampling process now introduces temporal dependence in the noise. We control this dependence through the following geometric-mixing condition on the behavior-induced chain.
\begin{assumption}
\label{assump:markovian-mixing}
The behavior-induced chain \(X_k=(s_k,a_k)\) is irreducible and aperiodic with stationary distribution \(d\), and it is initialized from stationarity, \(X_0\sim d\).  There exists \(t_{\mathrm{mix}}\ge1\) such that, for every \(\ell\ge0\),
\[
  \sup_{x\in\calS\times\calA}
  \left\|
  (P^b)^\ell(x,\cdot)-d
  \right\|_{\mathrm{TV}}
  \le
  2^{-\lfloor \ell/t_{\mathrm{mix}}\rfloor}.
\]
\end{assumption}

Under this condition, the same direct SLS representation is combined with a conservative Poisson-equation estimate for the Markovian noise. This yields the following finite-time bound.
\begin{theorem}
\label{thm:markovian-reference-filter-bound}
Let the single-trajectory Q-learning recursion
\cref{eq:markovian-single-trajectory-update-ref} be run with a constant step
size \(\alpha\in(0,1)\).  Let \cref{assump:basic} hold with \(d\) equal to the
stationary distribution of the behavior-induced chain \(X_k=(s_k,a_k)\), and
assume \cref{assump:markovian-mixing} holds.  Fix any
\(\varepsilon>0\) such that
\[
  \beta_\varepsilon:=\jsr(\calM_\alpha)+\varepsilon<1,
\]
let \(W_{\max}\) be the deterministic noise bound defined in Appendix~\ref{app:proof-bounded-noise}, \cref{eq:Wmax-general}, and let \(K_{\beta_\varepsilon}\) be defined in Appendix~\ref{app:direct-reference-filter-noise-floor}, \cref{lem:product-growth-constant-direct}. Then, for every \(k\ge0\),
\begin{equation}\label{eq:markovian-reference-filter-final-bound}
\begin{aligned}
\E[\|Q_k-Q^*\|_\infty]
\le\;&
\left(
1+
\frac{2\gamma d_{\max}}{d_{\min}(1-\gamma)}
\right)
128|\calS|^2|\calA|^2\sqrt{W_{\max}} \\
&\quad\times
\left(
t_{\mathrm{mix}}\sqrt{\frac{\alpha}{d_{\min}(1-\gamma)}}
+
\frac{\alpha t_{\mathrm{mix}}}{d_{\min}(1-\gamma)}
\right) \\
&+
K_{\beta_\varepsilon}\beta_\varepsilon^k\|Q_0-Q^*\|_\infty \\
&+
2\alpha\gamma d_{\max}K_{\beta_\varepsilon}^2
k\beta_\varepsilon^{k-1}\|Q_0-Q^*\|_\infty,
\end{aligned}
\end{equation}
where \(k\beta_\varepsilon^{k-1}=0\) when \(k=0\).
\end{theorem}
The proof is given in Appendix~\ref{sec:extensions-markovian-observations}. The first line of \cref{eq:markovian-reference-filter-final-bound} is the Markovian noise floor. The martingale term in the Poisson decomposition is controlled by a pathwise bound for the adaptive tables \(Q_i\); consequently, for fixed problem data and fixed \(t_{\mathrm{mix}}\), this conservative version is of order \(t_{\mathrm{mix}}\sqrt{\alpha}+\alpha t_{\mathrm{mix}}\), and hence it decreases to zero as \(\alpha\downarrow0\). The remaining two terms are transients governed by the JSR rate.

\section{Numerical illustration}
\label{sec:numerical-jsr-illustration}

This section gives a small numerical illustration of the JSR transient rate and its comparison with the standard row-sum rate.  We use the action-major ordering induced by the notation above, namely
\[
  (1,1),(2,1),(1,2),(2,2).
\]
Let us consider an i.i.d. asynchronous Q-learning problem with $\calS=\{1,2\}$, $\calA=\{1,2\}$, $\gamma=0.9$, state-action sampling distribution $d=(0.10,0.40,0.10,0.40)$ in this ordering, $D=\diag(d)$, and transition probabilities $P(1\mid s,a)=0.5$ and $P(2\mid s,a)=0.5$ for every $s\in\calS$ and $a\in\calA$. The reward function is $r(s,1,s')=1$, $r(s,2,s')=-9$ for all $s,s'\in\calS$.
Therefore, action $a=1$ is optimal in both states and $Q^*(s,1)=10$ and $Q^*(s,2)=0$ for all $s\in\calS$.
The initial condition is \(Q_0=0\). To make the constant step-size residual visible in the empirical trajectory, the simulation uses the bounded additive reward-observation model
\[
  r_{k+1}^{\mathrm{obs}}
  =
  r(s_k,a_k,s_{k+1})+\xi_k,
  \qquad
  \xi_k\in\{-0.1,0.1\},
\]
where the variables \(\xi_k\) are independent of the trajectory and have equal probabilities on the two displayed values.  Since \(\E[\xi_k]=0\), the averaged Bellman operator and the value \(Q^*\) computed above are unchanged.  This is a bounded reward-observation variant of \cref{assump:basic}; the pathwise bounded-noise estimates used in the theory apply verbatim after taking the reward bound to be at least \(R_{\max}=9.1\).  Setting \(\xi_k=0\) recovers the deterministic-reward model exactly. For each value of \(\alpha\), the empirical curve averages \(20\) independent trajectories and plots the relative error defined as $\frac{\|Q_k-Q^*\|_\infty}{\|Q_0-Q^*\|_\infty}$. The displayed direct rates are finite-product estimates computed by enumerating all products of these four matrices up to length \(L_{\max}=8\) and taking
\[
  \widehat\rho_{\mathrm{dir}}^{(8)}(\alpha)
  :=
  \max_{1\le \ell\le 8}\;
  \max_{\pi_1,\ldots,\pi_\ell\in\Theta}
  \rho\!
  \left(
  M_{\pi_\ell}\cdots M_{\pi_1}
  \right)^{1/\ell},
\]
where \(\rho(\cdot)\) denotes the ordinary spectral radius of a single matrix. This finite-product quantity is used as a numerical lower estimate of the direct JSR, rather than as a separate certificate of the exact JSR value.  The row-sum reference rate is computed exactly as $\rho_{\mathrm{row}}(\alpha)  =1-\alpha d_{\min}(1-\gamma)=1-0.01\alpha$. For \(\alpha=0.30,0.60,0.90\), the corresponding direct finite-product estimates are \(0.995377,0.990755,0.986132\), respectively, while the exact row-sum rates are \(0.997000,0.994000,0.991000\). The three figures plot the averaged empirical relative error together with two normalized reference lines. In all three cases, the finite-product direct-rate line decays faster than the row-sum line, reflecting the observed gap between \(\widehat\rho_{\mathrm{dir}}^{(8)}(\alpha)\) and \(\rho_{\mathrm{row}}(\alpha)\). Moreover, before the constant-step-size noise floor becomes dominant, the early decay of the Q-learning error closely follows the finite-product direct-rate line. This indicates that, for this example, the proposed direct-SLS rate estimate gives a useful description of the transient convergence behavior.

\begin{figure}[!ht]
  \centering
  \includegraphics[width=.64\linewidth]{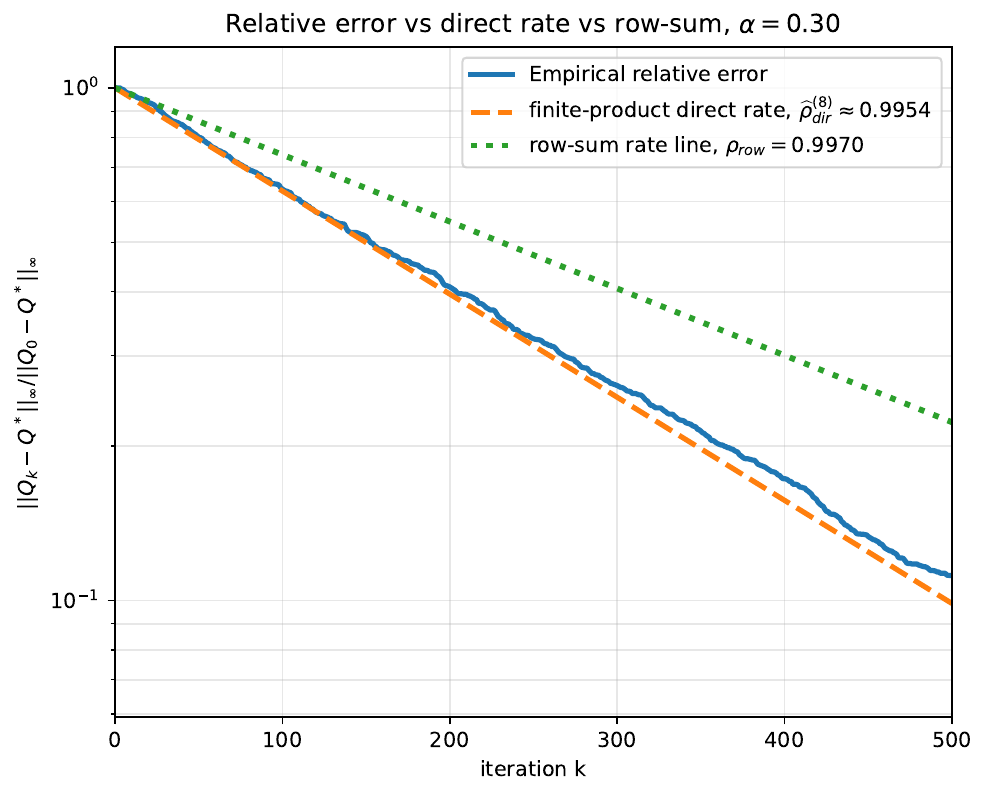}
  \caption{Relative Q-learning error for \(\alpha=0.30\), averaged over 20 independent seeds.  The empirical relative error is compared with the finite-product direct-rate line and the row-sum rate line, both normalized to start at one.}
  \label{fig:jsr-alpha-030}
\end{figure}

\begin{figure}[!ht]
  \centering
  \includegraphics[width=.64\linewidth]{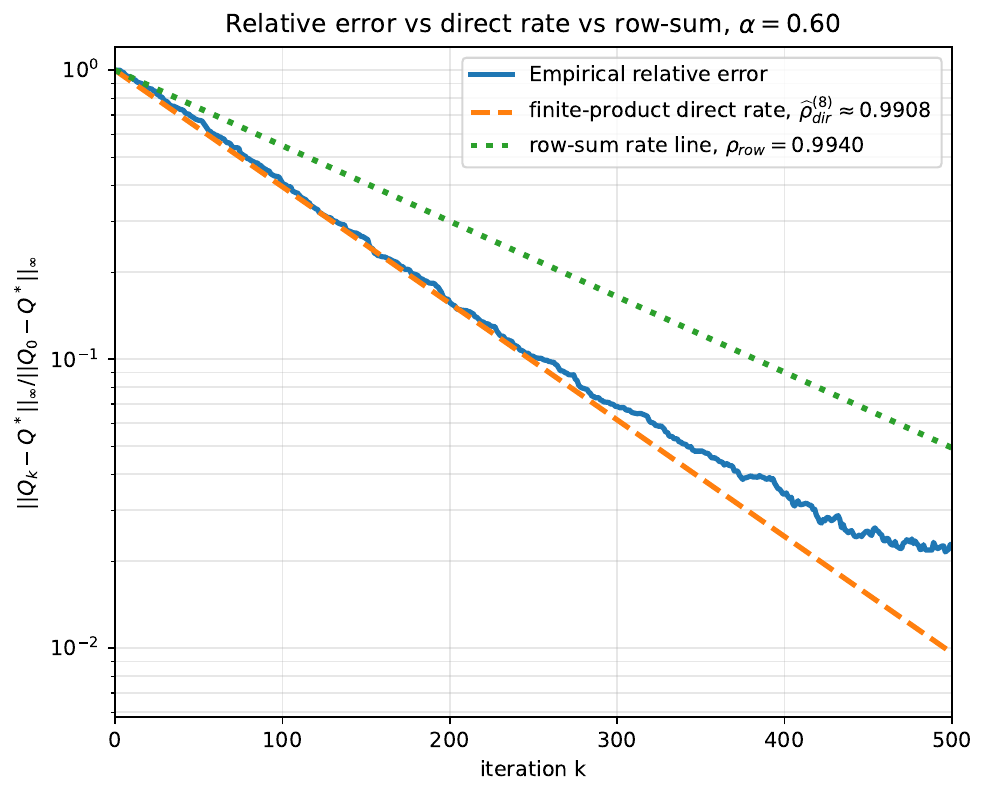}
  \caption{Relative Q-learning error for \(\alpha=0.60\), averaged over 20 independent seeds.  The finite-product direct-rate line tracks the transient decay more closely than the row-sum rate line in this example.}
  \label{fig:jsr-alpha-060}
\end{figure}

\begin{figure}[!ht]
  \centering
  \includegraphics[width=.64\linewidth]{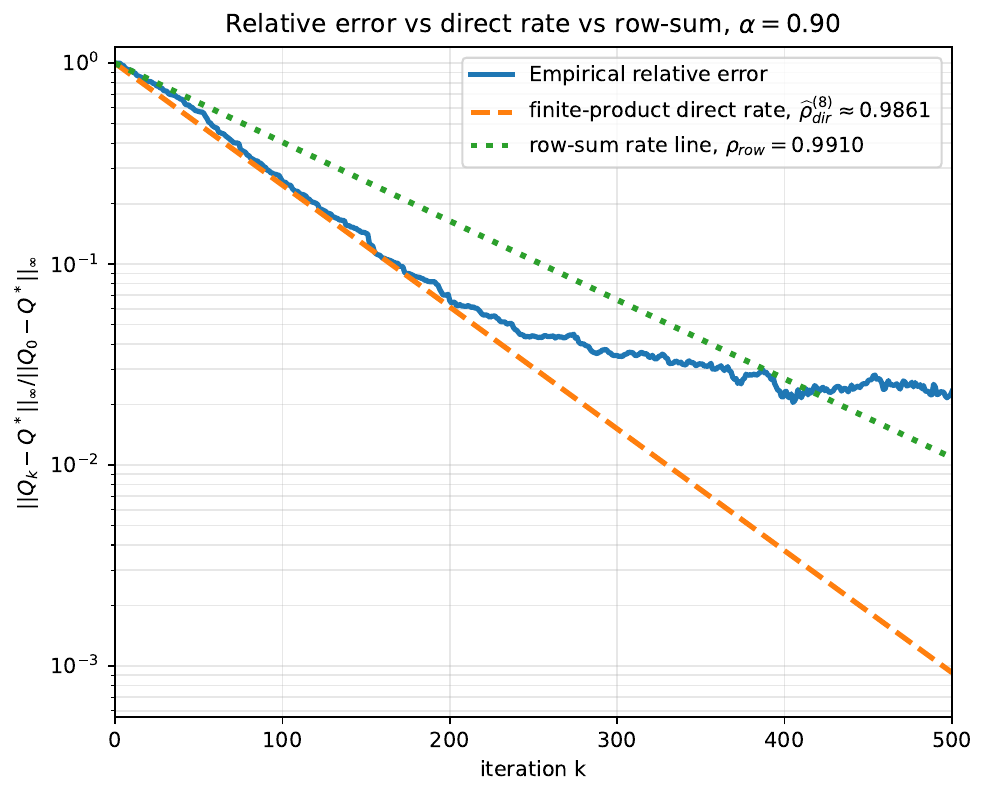}
  \caption{Relative Q-learning error for \(\alpha=0.90\), averaged over 20 independent seeds.  The constant-step-size reward-noise residual becomes visible as the empirical curve flattens.}
  \label{fig:jsr-alpha-090}
\end{figure}

\clearpage
\section{Conclusion}\label{sec:conclusion}
We developed a direct SLS finite-time analysis of Q-learning.  The central step is the exact stochastic-policy representation of the Bellman maximization error, which yields an SLS conditional-mean error recursion whose intrinsic deterministic rate is the direct JSR \(\jsr(\calM_\alpha)\).  The standard row-sum rate remains a simple upper bound, but it can be conservative.
The resulting bounds use a JSR-induced certificate for the direct drift.  The framework also extends to Markovian observations.  Overall, the direct-SLS viewpoint complements Bellman-contraction and comparison-system analyses by exposing instance-dependent drift rates and removing comparison-induced transient terms when sharper certificates are available.

\appendix
\begin{center}
{\Large\bfseries Appendix}
\end{center}
\renewcommand{\theHfigure}{\thesection.\arabic{figure}}

\section{Affine SLS representation of Q-learning}
\label{app:affine-switching-system-representation}
We next recall the affine SLS representation used in~\cite{lee2020unified,lee2023discrete,lee2024final}.  For each $Q\in\R^{|\calS|\,|\calA|}$, let $\pi_Q$ be a deterministic greedy policy satisfying $\pi_Q(s)\in\argmax_{a\in\calA}Q(s,a), s\in\calS$.
Then
\[
  V_Q=\Pi^{\pi_Q}Q,
  \qquad
  F(Q)=R+\gamma P\Pi^{\pi_Q}Q.
\]
Let $\pi_k:=\pi_{Q_k}$, and let $\pi^*$ be an optimal deterministic policy so that $Q^*=R+\gamma P\Pi^{\pi^*}Q^*$. Subtracting the fixed-point equation from the Q-learning recursion gives the stochastic affine SLS recursion
\[
(Q_{k+1}-Q^*)
=
\bigl(I-\alpha D+\alpha\gamma DP\Pi^{\pi_k}\bigr)(Q_k-Q^*)
+
\alpha\gamma DP(\Pi^{\pi_k}-\Pi^{\pi^*})Q^*
+
\alpha w_k,
\]
where
\[
w_k
:=
(e_{a_k}\otimes e_{s_k})
\left(
r_{k+1}+\gamma\max_{u\in\calA}Q_k(s'_k,u)-Q_k(s_k,a_k)
\right)
-
D(F(Q_k)-Q_k).
\]
Equivalently, it can be written as
\[
  (Q_{k+1}-Q^*)=A_{\pi_k}(Q_k-Q^*)+b_{\pi_k}+\alpha w_k,
\]
where $  A_{\pi}:=I-\alpha D+\alpha\gamma DP\Pi^\pi$ and $b_{\pi}:=\alpha\gamma DP(\Pi^\pi-\Pi^{\pi^*})Q^*$. The representation is affine because $b_{\pi_k}$ is generally nonzero whenever the greedy policy $\pi_k$ is not optimal. The next section shows that this offset is an artifact of the chosen parametrization: $V_{Q_k}-V^*$ can instead be represented directly as the action-wise error averaged under a suitable stochastic policy.

\section{Boundedness and noise variance}
\label{app:proof-bounded-noise}

\begin{lemma}\label{lem:bounded-noise}
Under \cref{assump:basic}, let $\{\calF_k\}_{k\ge0}$ be the natural filtration defined in \cref{subsec:iid-qlearning}.  In particular, $Q_0$ is deterministic.  The Q-learning iterates satisfy
\[
  \|Q_k\|_\infty\le \max\left\{\|Q_0\|_\infty,\frac{R_{\max}}{1-\gamma}\right\},
  \qquad \forall k\ge0.
\]
Moreover, for the noise
\[
w_k
:=
(e_{a_k}\otimes e_{s_k})
\left(
r_{k+1}+\gamma\max_{u\in\calA}Q_k(s'_k,u)-Q_k(s_k,a_k)
\right)
-
D(F(Q_k)-Q_k),
\]
we have
\[
  \E[\|w_k\|_2^2\mid\calF_k]
  \le
  W_{\max},
  \qquad \forall k\ge0,
\]
where $W_{\max}$ is defined as
\begin{equation}\label{eq:Wmax-general}
  W_{\max}
  :=
  \left(R_{\max}+(1+\gamma)
  \max\left\{\|Q_0\|_\infty,\frac{R_{\max}}{1-\gamma}\right\}\right)^2.
\end{equation}
\end{lemma}

\begin{proof}
We prove the two claims separately.
First, we prove boundedness by induction. At time $k=0$,
\[
  \|Q_0\|_\infty=\|Q_0\|_\infty\le \max\left\{\|Q_0\|_\infty,\frac{R_{\max}}{1-\gamma}\right\},
\]
so the claim is true initially. Suppose now that
$$\|Q_k\|_\infty\le \max\left\{\|Q_0\|_\infty,\frac{R_{\max}}{1-\gamma}\right\}$$
for some $k$. The random target used in the sampled coordinate is
\[
  Y_k:=r_{k+1}+\gamma\max_{u\in\calA}Q_k(s'_k,u).
\]
The induction hypothesis gives $|Q_k(s'_k,u)|\le \max\left\{\|Q_0\|_\infty,\frac{R_{\max}}{1-\gamma}\right\}$ for every $u$, and hence
\[
  |Y_k|
  \le
  R_{\max}+\gamma \max\left\{\|Q_0\|_\infty,\frac{R_{\max}}{1-\gamma}\right\}.
\]
Because $\max\left\{\|Q_0\|_\infty,\frac{R_{\max}}{1-\gamma}\right\}\ge R_{\max}/(1-\gamma)$, we have
\[
\begin{aligned}
  R_{\max}
  &\le
  (1-\gamma)
  \max\left\{
  \|Q_0\|_\infty,
  \frac{R_{\max}}{1-\gamma}
  \right\}, \\
  R_{\max}+\gamma
  \max\left\{
  \|Q_0\|_\infty,
  \frac{R_{\max}}{1-\gamma}
  \right\}
  &\le
  \max\left\{
  \|Q_0\|_\infty,
  \frac{R_{\max}}{1-\gamma}
  \right\}.
\end{aligned}
\]
Therefore, the target $Y_k$ always lies in the interval
\begin{equation}\label{eq:1}
  Y_k\in
  \left[-B_0,\,B_0\right],
  \qquad
  B_0:=
  \max\left\{
  \|Q_0\|_\infty,
  \frac{R_{\max}}{1-\gamma}
  \right\}.
\end{equation}
Only the coordinate $(s_k,a_k)$ is changed. For that coordinate,
\[
  Q_{k+1}(s_k,a_k)
  =
  (1-\alpha)Q_k(s_k,a_k)+\alpha Y_k.
\]
Since $0<\alpha<1$, this is a convex combination of two numbers in the interval in~\cref{eq:1}, and therefore it also belongs to the same interval in~\cref{eq:1}. Every other coordinate is unchanged and therefore remains in the same interval. This proves $\|Q_{k+1}\|_\infty\le \max\left\{\|Q_0\|_\infty,\frac{R_{\max}}{1-\gamma}\right\}$. By induction, we have
\[
  \|Q_k\|_\infty\le \max\left\{\|Q_0\|_\infty,\frac{R_{\max}}{1-\gamma}\right\},
  \qquad \forall k\ge0,
\]
which proves the first statement. We now prove the conditional second-moment bound for the noise. Let
\[
\zeta_k
:=
(e_{a_k}\otimes e_{s_k})
\left(
 r_{k+1}+\gamma\max_{u\in\calA}Q_k(s'_k,u)-Q_k(s_k,a_k)
\right).
\]
By definition, $w_k=\zeta_k-\E[\zeta_k\mid\calF_k]$.  For any square-integrable random vector $Z$ and any $\sigma$-field $\calF$,
\[
  \E[\|Z-\E[Z\mid\calF]\|_2^2\mid\calF]
  =
  \E[\|Z\|_2^2\mid\calF]-\|\E[Z\mid\calF]\|_2^2
  \le
  \E[\|Z\|_2^2\mid\calF].
\]
Applying this identity with $Z=\zeta_k$ gives
\[
  \E[\|w_k\|_2^2\mid\calF_k]
  \le
  \E[\|\zeta_k\|_2^2\mid\calF_k].
\]
The vector $\zeta_k$ has at most one nonzero coordinate. Therefore, its Euclidean norm is exactly the absolute value of the scalar temporal-difference term. By the boundedness just proved, one can derive
\[
\begin{aligned}
\left|
 r_{k+1}+\gamma\max_{u\in\calA}Q_k(s'_k,u)-Q_k(s_k,a_k)
\right|
&\le
R_{\max}+\gamma
\max\left\{
\|Q_0\|_\infty,
\frac{R_{\max}}{1-\gamma}
\right\} \\
&\quad+
\max\left\{
\|Q_0\|_\infty,
\frac{R_{\max}}{1-\gamma}
\right\} \\
&=
R_{\max}+(1+\gamma)
\max\left\{
\|Q_0\|_\infty,
\frac{R_{\max}}{1-\gamma}
\right\}.
\end{aligned}
\]
Consequently, we have
\[
  \|\zeta_k\|_2^2
  \le
  \left(R_{\max}+(1+\gamma)\max\left\{\|Q_0\|_\infty,\frac{R_{\max}}{1-\gamma}\right\}\right)^2
  =
  W_{\max}.
\]
Taking conditional expectation on both sides of the above inequality gives the desired bound.
\end{proof}

Note that since $Q_0$ is deterministic by \cref{assump:basic}, $W_{\max}$ is a deterministic finite constant.
Since \(\max\left\{\|Q_0\|_\infty,R_{\max}/(1-\gamma)\right\}\ge R_{\max}/(1-\gamma)\), the crude but useful bound on $W_{\max}$ is $W_{\max}\le 4\max\left\{\|Q_0\|_\infty,\frac{R_{\max}}{1-\gamma}\right\}^2$. In particular, under the common initialization convention
\(\|Q_0\|_\infty\le R_{\max}/(1-\gamma)\), one has
\begin{equation}\label{eq:Wmax-normalized-bound}
  W_{\max}\le
  \frac{4R_{\max}^2}{(1-\gamma)^2}.
\end{equation}
We will use the boundedness and noise estimate in~\cref{lem:bounded-noise} throughout the finite-time analysis.

\section{Stochastic-policy linearization of the Bellman max}
\label{app:proof-stochastic-policy-linearization}

\begin{lemma}\label{lem:stochastic-policy-linearization}
For every $Q\in\R^{|\calS|\,|\calA|}$, there exists a stochastic policy $\mu_Q:\calS\to\Delta_{|\calA|}$ such that
\begin{equation*}
  V_Q-V^*=\Pi^{\mu_Q}(Q-Q^*).
\end{equation*}
Moreover, $\mu_Q$ can be chosen as a measurable function of $Q$.
\end{lemma}

\begin{proof}
Fix a state $s\in\calS$ and write
\[
  e:=Q-Q^*.
\]
We first show that the scalar number
\[
  y_s:=V_Q(s)-V^*(s)
\]
lies between the smallest and largest action-wise errors at state $s$.
Since
\[
V_Q(s)
=
\max_{a\in\calA}\{Q^*(s,a)+e(s,a)\},
\qquad
V^*(s)=\max_{a\in\calA}Q^*(s,a),
\]
we have the upper bound
\[
\begin{aligned}
V_Q(s)
&=\max_{a\in\calA}\{Q^*(s,a)+e(s,a)\}  \\
&\le
\max_{a\in\calA}Q^*(s,a)
+
\max_{a\in\calA}e(s,a)
=
V^*(s)+\max_{a\in\calA}e(s,a),
\end{aligned}
\]
which leads to
\[
  y_s\le \max_{a\in\calA}e(s,a).
\]
For the lower bound, let $a^*(s)$ be a maximizer of $Q^*(s,\cdot)$.  Then
\[
\begin{aligned}
V_Q(s)=\max_{a\in\calA}\{Q^*(s,a)+e(s,a)\} \ge Q^*(s,a^*(s))+e(s,a^*(s))\ge V^*(s)+\min_{a\in\calA}e(s,a).
\end{aligned}
\]
Hence, it follows that
\[
  \min_{a\in\calA}e(s,a)
  \le y_s
  \le \max_{a\in\calA}e(s,a).
\]
In one dimension, the convex hull of the finite set
$\{e(s,a):a\in\calA\}$ is exactly the interval between its minimum and maximum.  Therefore $y_s$ can be written as a convex combination of the action-wise errors.
We now choose one such convex combination explicitly.  Break ties lexicographically and define
\[
  a_{\min}(s)\in\argmin_{a\in\calA}e(s,a),
  \qquad
  a_{\max}(s)\in\argmax_{a\in\calA}e(s,a).
\]
If
\[
  e(s,a_{\min}(s))=e(s,a_{\max}(s)),
\]
then all action-wise errors at state $s$ are equal.  In this case $y_s$ must equal that common value, so we may set
\[
  \mu_Q(a_{\min}(s)\mid s)=1.
\]
If the two values are different, define
\[
  \lambda_s
  :=
  \frac{y_s-e(s,a_{\min}(s))}
       {e(s,a_{\max}(s))-e(s,a_{\min}(s))}.
\]
The bounds above guarantee $0\le\lambda_s\le1$.  Set
\[
  \mu_Q(a_{\max}(s)\mid s)=\lambda_s,
  \qquad
  \mu_Q(a_{\min}(s)\mid s)=1-\lambda_s,
\]
and set all other action probabilities to zero.  Then
\[
\sum_{a\in\calA}\mu_Q(a\mid s)e(s,a)
=
(1-\lambda_s)e(s,a_{\min}(s))
+
\lambda_s e(s,a_{\max}(s))
=
y_s.
\]
This proves the desired identity at the fixed state $s$.
Repeating the same construction for every state and stacking the resulting scalar identities gives the exact vector--matrix identity
\[
\begin{aligned}
V_Q-V^*
&=
\begin{bmatrix}
V_Q(1)-V^*(1)\\
\vdots\\
V_Q(|\calS|)-V^*(|\calS|)
\end{bmatrix} \\
&=
\begin{bmatrix}
\sum_{a\in\calA}\mu_Q(a\mid 1)e(1,a)\\
\vdots\\
\sum_{a\in\calA}\mu_Q(a\mid |\calS|)e(|\calS|,a)
\end{bmatrix} \\
&=
\begin{bmatrix}
\mu_Q(1)^\top\otimes e_1^\top\\
\mu_Q(2)^\top\otimes e_2^\top\\
\vdots\\
\mu_Q(|\calS|)^\top\otimes e_{|\calS|}^\top
\end{bmatrix}(Q-Q^*) \\
&=\Pi^{\mu_Q}(Q-Q^*).
\end{aligned}
\]
Finally, the above rule can be made deterministic and unique as a function of $Q$.  Because the action set is finite, lexicographic tie-breaking selects exactly one minimizer and exactly one maximizer at each state.  Once these two actions are fixed, the displayed formula determines the coefficient $\lambda_s$ uniquely; in the equal-error case, the rule deterministically assigns all mass to the selected action.  Thus, for each input $Q$, the construction returns one well-defined policy rather than a set of possible choices.  Since this deterministic rule is implemented by finitely many comparisons of continuous functions of $Q$ together with the explicit formula for $\lambda_s$, the resulting policy is measurable as a function of $Q$.
\end{proof}

\section{Proof of~\Cref{thm:direct-representation}}
\label{app:proof-direct-representation}

Start from the vector Q-learning recursion
\[
  Q_{k+1}=Q_k+\alpha\{D(F(Q_k)-Q_k)+w_k\}.
\]
Subtract $Q^*$ from both sides.  Since $Q^*=F(Q^*)$, we obtain
\[
\begin{aligned}
 (Q_{k+1}-Q^*)
 &= Q_{k+1}-Q^* \\
 &= Q_k-Q^*
 +\alpha D(F(Q_k)-Q_k) +\alpha w_k \\
 &= (Q_k-Q^*)
 +\alpha D\{F(Q_k)-F(Q^*)-(Q_k-Q^*)\}
 +\alpha w_k .
\end{aligned}
\]
This identity is purely algebraic; no approximation has been made.
Next, we use the form of the Bellman optimality operator,
\[
  F(Q)=R+\gamma PV_Q,
\]
which leads to
\[
  F(Q_k)-F(Q^*)
  =
  \gamma P(V_{Q_k}-V^*).
\]
Because $Q_k$ is $\calF_k$-measurable and the policy in
\cref{lem:stochastic-policy-linearization} can be selected measurably as a function of its input, there exists an $\calF_k$-measurable stochastic policy
\[
  \mu_k:=\mu_{Q_k}
\]
such that
\[
  V_{Q_k}-V^* = \Pi^{\mu_k}(Q_k-Q^*).
\]
Substituting this into the previous display gives
\[
  F(Q_k)-F(Q^*)
  =
  \gamma P\Pi^{\mu_k}(Q_k-Q^*).
\]
Hence
\[
\begin{aligned}
 (Q_{k+1}-Q^*)
 &=(Q_k-Q^*)+
 \alpha D\{\gamma P\Pi^{\mu_k}(Q_k-Q^*)-(Q_k-Q^*)\}
 +\alpha w_k \\
 &=
 \left(I-\alpha D+\alpha\gamma DP\Pi^{\mu_k}\right)(Q_k-Q^*)
 +\alpha w_k \\
 &=M_{\mu_k}(Q_k-Q^*)+\alpha w_k.
\end{aligned}
\]
This proves the direct SLS representation.
It remains only to identify the conditional mean.  By construction in
\cref{subsec:iid-qlearning},
\[
  \E[w_k\mid\calF_k]=0.
\]
Since $M_{\mu_k}$ and $Q_k-Q^*$ are $\calF_k$-measurable, taking conditional expectation in the last recursion yields
\[
  \E[Q_{k+1}-Q^*\mid\calF_k]
  =
  M_{\mu_k}(Q_k-Q^*).
\]
This completes the proof.

\section{Proof of~\Cref{lem:convex-hull-M}}
\label{app:proof-convex-hull-M}

Fix a stochastic policy \(\mu\).  A deterministic policy \(\pi\in\Theta\) chooses one action at every state.  Define
\[
  c_\pi(\mu):=\prod_{s\in\calS}\mu(\pi(s)\mid s),
  \qquad \pi\in\Theta.
\]
Then \(c_\pi(\mu)\ge0\) and \(\sum_{\pi\in\Theta}c_\pi(\mu)=1\).  Moreover, for every \(s\in\calS\) and \(a\in\calA\),
\[
  \sum_{\pi\in\Theta:\,\pi(s)=a}c_\pi(\mu)=\mu(a\mid s),
\]
which implies
\[
  \Pi^\mu=\sum_{\pi\in\Theta}c_\pi(\mu)\Pi^\pi.
\]
By the linear dependence of \(M_\mu=I-\alpha D+\alpha\gamma DP\Pi^\mu\) on \(\Pi^\mu\),
\[
  M_\mu
  =
  \sum_{\pi\in\Theta}c_\pi(\mu)M_\pi.
\]
Thus \(M_\mu\in\co(\calM_\alpha)\).

It remains to compare the JSRs.  We use the standard convex-hull invariance property of the JSR: for any bounded family of matrices, taking the convex hull does not change its joint spectral radius~\cite[Section~1.2.2.7]{jungers2009joint}; see also the original JSR theory in~\cite{rota1960note}.  Applying this property to the finite family \(\calM_\alpha\) gives
\[
  \jsr(\co(\calM_\alpha))=\jsr(\calM_\alpha),
\]
which completes the proof.

\section{Proof of~\Cref{lem:row-sum-jsr-upper-bound}}
\label{app:proof-row-sum-jsr-upper-bound}

Note that the matrices \(D\), \(P\), and \(\Pi^\mu\) are nonnegative. Moreover, \(1-\alpha d(s,a)\ge0\) because \(0<\alpha<1\) and \(d(s,a)\le 1\). Hence every entry of \(M_\mu\) is nonnegative. The row of \(P\Pi^\mu\) associated with \((s,a)\) is a probability vector, because after taking action \(a\) in state \(s\), the next state is sampled according to \(P(\cdot\mid s,a)\) and the next action is sampled according to \(\mu(\cdot\mid s')\). Therefore, that row sums to one, and the corresponding row of \(M_\mu\) sums to
\[
  1-\alpha d(s,a)+\alpha\gamma d(s,a)
  =
  1-\alpha d(s,a)(1-\gamma)
  \le
  1-\alpha d_{\min}(1-\gamma)
  =
  \rho_{\mathrm{row}}.
\]
For a nonnegative matrix, the induced infinity norm is the maximum row sum. Thus \(\|M_\mu\|_\infty\le\rho_{\mathrm{row}}\) for every stochastic policy \(\mu\), and in particular for every deterministic policy. Since the JSR of a family is bounded above by any common induced-norm bound on that family, \(\jsr(\calM_\alpha)\le\rho_{\mathrm{row}}\).

\section{Proof of~\Cref{prop:direct-rate}}
\label{app:proof-direct-rate}

Let
\[
  \beta_\varepsilon:=\jsr(\calM_\alpha)+\varepsilon.
\]
By assumption, $\beta_\varepsilon<1$.  A standard extremal-norm consequence of the JSR theory says that, for every number strictly larger than the JSR of a finite matrix family, there exists a norm that contracts all matrices in the family by that number~\cite{rota1960note,jungers2009joint}. Applying this result to the finite family $\calM_\alpha$ gives a norm $p_\varepsilon$ satisfying
\[
  p_\varepsilon(M_\pi x)
  \le
  \beta_\varepsilon p_\varepsilon(x),
  \qquad
  \forall x\in\R^{|\calS|\,|\calA|},
  \forall\pi\in\Theta.
\]

Now take an arbitrary stochastic policy $\mu$.  By \cref{lem:convex-hull-M},
\[
  M_\mu=\sum_{\pi\in\Theta}c_\pi(\mu)M_\pi,
  \qquad
  c_\pi(\mu)\ge0,
  \qquad
  \sum_{\pi\in\Theta}c_\pi(\mu)=1.
\]
Using the triangle inequality and positive homogeneity of the norm,
\[
\begin{aligned}
  p_\varepsilon(M_\mu x)
  =
  p_\varepsilon\left(\sum_{\pi\in\Theta}c_\pi(\mu)M_\pi x\right)
  \le
  \sum_{\pi\in\Theta}c_\pi(\mu)p_\varepsilon(M_\pi x)
  \le
  \sum_{\pi\in\Theta}c_\pi(\mu)\beta_\varepsilon p_\varepsilon(x)
  =
  \beta_\varepsilon p_\varepsilon(x).
\end{aligned}
\]
As a result, the same norm controls every stochastic-policy mode.  Now let \(x_{k+1}=M_{\mu_k}x_k\) be any trajectory of \cref{eq:direct-discrete-switched-system}.  Repeatedly applying the preceding inequality gives
\[
  p_\varepsilon(x_k)
  \le
  \beta_\varepsilon^k p_\varepsilon(x_0),
  \qquad k\ge0.
\]
Because \(\beta_\varepsilon<1\) and \(p_\varepsilon\) is a norm on a finite-dimensional space, this Lyapunov estimate implies \(x_k\to0\) exponentially.  This completes the proof.

\section{JSR-based Lyapunov function construction}
\label{app:jsr-lyapunov-aux}

\begin{lemma}
\label{lem:common-lyapunov-restricted-family}
Let
\[
\calH:=\{A_1,A_2,\ldots,A_M\}\subset\R^{m\times m}.
\]
Fix any \(\varepsilon>0\) such that
\[
\beta_\varepsilon:=\jsr(\calH)+\varepsilon\in(0,1).
\]
For a sequence of switching modes \(\sigma=(\sigma_1,\ldots,\sigma_k)\in\{1,\ldots,M\}^k\), write
\[
A_\sigma:=A_{\sigma_k}\cdots A_{\sigma_1},
\]
and for \(k=0\) interpret \(\{1,\ldots,M\}^0\) as the singleton empty word and \(A_\sigma=I\). For each integer \(t\ge0\), define the piecewise quadratic function
\begin{equation*}
V_\varepsilon^t(x)
:=
\sum_{k=0}^{t}
\beta_\varepsilon^{-2k}
\max_{\sigma\in\{1,2,\ldots,M\}^k}
\|A_\sigma x\|_2^2,
\qquad x\in\R^m.
\end{equation*}
Then the following statements hold.
\begin{enumerate}
\item For every \(t\ge0\),
\begin{equation}\label{eq:generic-finite-recursion}
V_\varepsilon^{t+1}(x)
\ge
\|x\|_2^2
+
\beta_\varepsilon^{-2}
\max_{i\in\{1,\ldots,M\}}
V_\varepsilon^t(A_i x),
\qquad
\forall x\in\R^m.
\end{equation}

\item For every \(t\ge0\), every \(x\in\R^m\), and every \(\lambda\in\R\),
\[
V_\varepsilon^t(\lambda x)=|\lambda|^2V_\varepsilon^t(x),
\]
and
\[
V_\varepsilon^t(x)\le V_\varepsilon^{t+1}(x).
\]

\item There exists a constant \(C_\varepsilon>0\) such that
\begin{equation}\label{eq:generic-jsr-equivalence-finite}
\|x\|_2^2
\le
V_\varepsilon^t(x)
\le
C_\varepsilon\|x\|_2^2,
\qquad
\forall x\in\R^m,\quad \forall t\ge0.
\end{equation}

\item For every \(x\in\R^m\), the limit
\[
V_\varepsilon^\infty(x):=\lim_{t\to\infty}V_\varepsilon^t(x)
\]
exists and is finite.  Moreover,
\begin{equation}\label{eq:generic-jsr-equivalence-infty}
\|x\|_2^2
\le
V_\varepsilon^\infty(x)
\le
C_\varepsilon\|x\|_2^2,
\qquad
\forall x\in\R^m.
\end{equation}

\item The function
\[
p_\varepsilon(x):=\sqrt{V_\varepsilon^\infty(x)}
\]
is a norm on \(\R^m\).

\item The function \(V_\varepsilon^\infty\) satisfies the stronger Lyapunov inequality
\begin{equation*}
V_\varepsilon^\infty(A_i x)
\le
\beta_\varepsilon^2
\left(
V_\varepsilon^\infty(x)-\|x\|_2^2
\right)
\le
\beta_\varepsilon^2V_\varepsilon^\infty(x),
\qquad
\forall x\in\R^m,\quad
\forall i\in\{1,\ldots,M\}.
\end{equation*}
Equivalently,
\[
p_\varepsilon(A_i x)
\le
\beta_\varepsilon p_\varepsilon(x),
\qquad
\forall x\in\R^m,\quad
\forall i\in\{1,\ldots,M\}.
\]

\item Consequently,
\[
\|A_i\|_{p_\varepsilon}
\le
\beta_\varepsilon,
\qquad
\forall i\in\{1,\ldots,M\},
\]
where \(\|A_i\|_{p_\varepsilon}\) is the induced matrix norm generated by \(p_\varepsilon\),
\[
\|A_i\|_{p_\varepsilon}
:=
\sup_{x\in\R^m,\ x\ne0}\frac{p_\varepsilon(A_i x)}{p_\varepsilon(x)}.
\]
Therefore, the induced-norm bound yields
\[
\jsr(\calH)\le \beta_\varepsilon.
\]
\end{enumerate}
\end{lemma}

\begin{proof}
We prove the statements one by one.

\medskip
\noindent
\textit{Proof of 1).}
By definition,
\[
V_\varepsilon^{t+1}(x)
=
\sum_{k=0}^{t+1}
\beta_\varepsilon^{-2k}
\max_{\sigma\in\{1,\ldots,M\}^k}
\|A_\sigma x\|_2^2.
\]
The \(k=0\) term is \(\|x\|_2^2\).  For \(k\ge1\), write \(k=j+1\).  A word of length \(j+1\) may be described by its first applied index \(i\in\{1,\ldots,M\}\) and the remaining word \(\tau\in\{1,\ldots,M\}^j\).  Therefore, decomposing by the first applied index gives
\begin{align*}
V_\varepsilon^{t+1}(x)
&=
\|x\|_2^2
+
\sum_{j=0}^{t}
\beta_\varepsilon^{-2(j+1)}
\max_{i\in\{1,\ldots,M\}}
\max_{\tau\in\{1,\ldots,M\}^j}
\|A_\tau A_i x\|_2^2\\
&=
\|x\|_2^2
+
\beta_\varepsilon^{-2}
\sum_{j=0}^{t}
\beta_\varepsilon^{-2j}
\max_{i\in\{1,\ldots,M\}}
\max_{\tau\in\{1,\ldots,M\}^j}
\|A_\tau A_i x\|_2^2.
\end{align*}
Since the summands are nonnegative, the sum of the maxima is at least the maximum of the sums.  Hence
\begin{align*}
V_\varepsilon^{t+1}(x)
&\ge
\|x\|_2^2
+
\beta_\varepsilon^{-2}
\max_{i\in\{1,\ldots,M\}}
\sum_{j=0}^{t}
\beta_\varepsilon^{-2j}
\max_{\tau\in\{1,\ldots,M\}^j}
\|A_\tau A_i x\|_2^2\\
&=
\|x\|_2^2
+
\beta_\varepsilon^{-2}
\max_{i\in\{1,\ldots,M\}}
V_\varepsilon^t(A_i x).
\end{align*}
This proves \cref{eq:generic-finite-recursion}.

\medskip
\noindent
\textit{Proof of 2).}
Absolute homogeneity follows from the homogeneity of the Euclidean norm:
\[
V_\varepsilon^t(\lambda x)
=
\sum_{k=0}^{t}
\beta_\varepsilon^{-2k}
\max_{\sigma\in\{1,\ldots,M\}^k}
\|A_\sigma(\lambda x)\|_2^2
=
|\lambda|^2V_\varepsilon^t(x).
\]
The monotonicity \(V_\varepsilon^t(x)\le V_\varepsilon^{t+1}(x)\) holds because \(V_\varepsilon^{t+1}\) contains all terms of \(V_\varepsilon^t\) plus one additional nonnegative term.

\medskip
\noindent
\textit{Proof of 3).}
The lower bound is immediate from the \(k=0\) term:
\[
V_\varepsilon^t(x)\ge\|x\|_2^2.
\]
For the upper bound, choose \(\eta\) such that
\[
\jsr(\calH)<\eta<\beta_\varepsilon.
\]
By the definition of the JSR, there exists an integer \(K\ge1\) such that
\[
\max_{\sigma\in\{1,\ldots,M\}^k}
\|A_\sigma\|_2^{1/k}
\le
\eta,
\qquad
\forall k\ge K.
\]
Thus, for all \(k\ge K\),
\[
\max_{\sigma\in\{1,\ldots,M\}^k}
\|A_\sigma\|_2
\le
\eta^k.
\]
Now define
\[
C_0:=
\max\left\{1,\,
\max_{0\le k\le K-1}
\eta^{-k}
\max_{\sigma\in\{1,\ldots,M\}^k}
\|A_\sigma\|_2
\right\}.
\]
Then, for every \(k\ge0\),
\[
\max_{\sigma\in\{1,\ldots,M\}^k}
\|A_\sigma\|_2
\le
C_0\eta^k.
\]
Therefore, the product bound gives the estimate
\begin{align*}
V_\varepsilon^t(x)
&=
\sum_{k=0}^{t}
\beta_\varepsilon^{-2k}
\max_{\sigma\in\{1,\ldots,M\}^k}
\|A_\sigma x\|_2^2\\
&\le
\sum_{k=0}^{t}
\beta_\varepsilon^{-2k}
\left(
\max_{\sigma\in\{1,\ldots,M\}^k}
\|A_\sigma\|_2
\right)^2
\|x\|_2^2\\
&\le
C_0^2
\sum_{k=0}^{t}
\left(\frac{\eta}{\beta_\varepsilon}\right)^{2k}
\|x\|_2^2\\
&\le
C_0^2
\sum_{k=0}^{\infty}
\left(\frac{\eta}{\beta_\varepsilon}\right)^{2k}
\|x\|_2^2.
\end{align*}
Since \(\eta/\beta_\varepsilon<1\), the geometric series converges.  Setting
\[
C_\varepsilon:=
\frac{C_0^2}{1-(\eta/\beta_\varepsilon)^2}
\]
gives \cref{eq:generic-jsr-equivalence-finite}.

\medskip
\noindent
\textit{Proof of 4).}
By 2), the sequence \(V_\varepsilon^t(x)\) is nondecreasing in \(t\).  By 3), it is bounded above by \(C_\varepsilon\|x\|_2^2\).  Hence
\[
V_\varepsilon^\infty(x):=\lim_{t\to\infty}V_\varepsilon^t(x)
\]
exists and is finite.  Passing to the limit in the bounds of 3) proves \cref{eq:generic-jsr-equivalence-infty}.

\medskip
\noindent
\textit{Proof of 5).}
For each \(k\ge1\), define
\[
\nu_k(x):=
\beta_\varepsilon^{-k}
\max_{\sigma\in\{1,\ldots,M\}^k}
\|A_\sigma x\|_2,
\]
and define \(\nu_0(x):=\|x\|_2\).  Each \(\nu_k\) is a seminorm because it is the pointwise maximum of seminorms, and \(\nu_0\) is a norm.  For finite \(t\),
\[
V_\varepsilon^t(x)=\sum_{k=0}^{t}\nu_k(x)^2,
\qquad
p_\varepsilon^t(x):=\sqrt{V_\varepsilon^t(x)}
=
\left(\sum_{k=0}^{t}\nu_k(x)^2\right)^{1/2}.
\]
The function \(p_\varepsilon^t\) is a norm.  Positivity follows from the \(\nu_0\) term, absolute homogeneity follows from the homogeneity of each \(\nu_k\), and the triangle inequality follows from Minkowski's inequality applied to the vector
\[
(\nu_0(x),\nu_1(x),\ldots,\nu_t(x)).
\]
Now
\[
p_\varepsilon(x)=\sqrt{V_\varepsilon^\infty(x)}
=
\lim_{t\to\infty}p_\varepsilon^t(x),
\]
and the limit is finite by 4).  Since
\[
p_\varepsilon^t(x+y)
\le
p_\varepsilon^t(x)+p_\varepsilon^t(y)
\]
for every \(t\), letting \(t\to\infty\) gives
\[
p_\varepsilon(x+y)
\le
p_\varepsilon(x)+p_\varepsilon(y).
\]
Absolute homogeneity is inherited from 2), and positive definiteness follows from
\[
p_\varepsilon(x)^2=V_\varepsilon^\infty(x)\ge\|x\|_2^2.
\]
Thus \(p_\varepsilon\) is a norm.

\medskip
\noindent
\textit{Proof of 6).}
Using 1),
\[
V_\varepsilon^{t+1}(x)
\ge
\|x\|_2^2
+
\beta_\varepsilon^{-2}
\max_{i\in\{1,\ldots,M\}} V_\varepsilon^t(A_i x).
\]
Therefore, rearranging the finite-horizon inequality gives
\[
\max_{i\in\{1,\ldots,M\}} V_\varepsilon^t(A_i x)
\le
\beta_\varepsilon^2
\left(V_\varepsilon^{t+1}(x)-\|x\|_2^2\right).
\]
Fixing \(i\), we have
\[
V_\varepsilon^t(A_i x)
\le
\beta_\varepsilon^2
\left(V_\varepsilon^{t+1}(x)-\|x\|_2^2\right).
\]
Letting \(t\to\infty\) and using monotone convergence of both sides gives
\[
V_\varepsilon^\infty(A_i x)
\le
\beta_\varepsilon^2
\left(V_\varepsilon^\infty(x)-\|x\|_2^2\right)
\le
\beta_\varepsilon^2V_\varepsilon^\infty(x).
\]
Taking square roots yields
\[
p_\varepsilon(A_i x)
\le
\beta_\varepsilon p_\varepsilon(x).
\]

\medskip
\noindent
\textit{Proof of 7).}
From 6),
\[
\|A_i\|_{p_\varepsilon}
=
\sup_{x\in\R^m,\ x\ne0}\frac{p_\varepsilon(A_i x)}{p_\varepsilon(x)}
\le
\beta_\varepsilon.
\]
Hence, for any word \(\sigma=(\sigma_1,\ldots,\sigma_k)\),
\[
\|A_\sigma\|_{p_\varepsilon}
\le
\prod_{j=1}^{k}\|A_{\sigma_j}\|_{p_\varepsilon}
\le
\beta_\varepsilon^k.
\]
Taking the maximum over all words of length \(k\), then the \(k\)th root, and finally the limit as \(k\to\infty\), gives
\[
\jsr(\calH)\le\beta_\varepsilon.
\]

This completes the proof.
\end{proof}

For the direct Q-learning family
\[
\calM_\alpha
=
\{M_\pi:\pi\in\Theta\},
\qquad
M_\pi=I-\alpha D+\alpha\gamma DP\Pi^\pi,
\]
fix \(\varepsilon>0\) such that
\[
\beta_\varepsilon:=\jsr(\calM_\alpha)+\varepsilon<1.
\]
Since \(\Theta\) is finite, enumerate \(\calM_\alpha\) as
\(\{M_1,\ldots,M_M\}\), where \(M=|\Theta|\).  For each integer
\(t\ge0\), define
\begin{equation}\label{eq:Veps-t-direct}
V_\varepsilon^t(x)
:=
\sum_{\ell=0}^{t}
\beta_\varepsilon^{-2\ell}
\max_{\pi_1,\ldots,\pi_\ell\in\Theta}
\left\|
M_{\pi_\ell}\cdots M_{\pi_1}x
\right\|_2^2,
\qquad x\in\R^{|\calS|\,|\calA|},
\end{equation}
where the term for \(\ell=0\) uses the empty product and is equal to
\(\|x\|_2^2\).  Define
\begin{equation}\label{eq:Veps-infty-direct}
V_\varepsilon^\infty(x)
:=
\lim_{t\to\infty}V_\varepsilon^t(x).
\end{equation}
By \cref{lem:common-lyapunov-restricted-family}, the limit exists and there
exists a constant \(C_\varepsilon\ge1\) such that
\begin{equation}\label{eq:Veps-equivalence}
\|x\|_2^2
\le
V_\varepsilon^\infty(x)
\le
C_\varepsilon\|x\|_2^2,
\qquad
\forall x\in\R^{|\calS|\,|\calA|}.
\end{equation}

\begin{lemma}
\label{lem:Veps-existence}\label{lem:Veps-jsr-direct}
With the definitions in \cref{eq:Veps-t-direct,eq:Veps-infty-direct}, the
norm-equivalence statement in \cref{eq:Veps-equivalence} holds.  In
particular,
\[
  p_\varepsilon(x):=\sqrt{V_\varepsilon^\infty(x)}
\]
is a norm on \(\R^{|\calS|\,|\calA|}\).  Moreover, for every deterministic policy \(\pi\in\Theta\),
\begin{equation}\label{eq:Veps-recursion-inequality}
V_\varepsilon^\infty(M_\pi x)
\le
\beta_\varepsilon^2
\left(
V_\varepsilon^\infty(x)-\|x\|_2^2
\right),
\qquad
\forall x\in\R^{|\calS|\,|\calA|}.
\end{equation}
Equivalently,
\[
p_\varepsilon(M_\pi x)
\le
\beta_\varepsilon p_\varepsilon(x),
\qquad
\forall x\in\R^{|\calS|\,|\calA|},\quad
\forall \pi\in\Theta.
\]
\end{lemma}

\begin{proof}
For the direct-family statement, apply the generic construction to the enumerated family
\(\calM_\alpha=\{M_1,\ldots,M_M\}\).  The finite-horizon inequality in item 1 gives
\[
V_\varepsilon^{t+1}(x)
\ge
\|x\|_2^2
+
\beta_\varepsilon^{-2}
\max_{\pi\in\Theta}V_\varepsilon^t(M_\pi x).
\]
Therefore, for each fixed \(\pi\),
\[
V_\varepsilon^t(M_\pi x)
\le
\beta_\varepsilon^2
\left(V_\varepsilon^{t+1}(x)-\|x\|_2^2\right).
\]
Letting \(t\to\infty\) gives \cref{eq:Veps-recursion-inequality}.  Taking square roots gives the stated norm contraction.

This completes the proof.
\end{proof}

\begin{lemma}
\label{lem:strong-convex-hull-extension}
For every stochastic policy \(\mu\),
\begin{equation}\label{eq:strong-convex-hull-extension}
V_\varepsilon^\infty(M_\mu x)
\le
\beta_\varepsilon^2
\left(
V_\varepsilon^\infty(x)-\|x\|_2^2
\right),
\qquad
\forall x\in\R^{|\calS|\,|\calA|}.
\end{equation}
\end{lemma}

\begin{proof}
For finite \(t\), each summand in \cref{eq:Veps-t-direct} is the maximum of convex quadratic functions, and a nonnegative sum of convex functions is convex.  Hence \(V_\varepsilon^t\) is convex for every \(t\).  Since
\[
V_\varepsilon^\infty(x)=\sup_{t\ge0}V_\varepsilon^t(x),
\]
the function \(V_\varepsilon^\infty\) is convex as well.
Next, fix a stochastic policy \(\mu\). By \cref{lem:convex-hull-M}, we have
\[
M_\mu=\sum_{\pi\in\Theta}c_\pi(\mu)M_\pi,
\qquad
c_\pi(\mu)\ge0,
\qquad
\sum_{\pi\in\Theta}c_\pi(\mu)=1.
\]
Applying Jensen's inequality to the convex function \(V_\varepsilon^\infty\), we have
\begin{align*}
V_\varepsilon^\infty(M_\mu x)
=
V_\varepsilon^\infty\left(
\sum_{\pi\in\Theta}c_\pi(\mu)M_\pi x
\right)
\le
\sum_{\pi\in\Theta}c_\pi(\mu)V_\varepsilon^\infty(M_\pi x)
\le
\max_{\pi\in\Theta}V_\varepsilon^\infty(M_\pi x).
\end{align*}
Using~\cref{eq:Veps-recursion-inequality},
\[
\max_{\pi\in\Theta}V_\varepsilon^\infty(M_\pi x)
\le
\beta_\varepsilon^2
\left(
V_\varepsilon^\infty(x)-\|x\|_2^2
\right),
\]
which proves \cref{eq:strong-convex-hull-extension}.
\end{proof}

\section{Proof of~\Cref{thm:finite-time-Veps}}
\label{app:proof-finite-time-Veps}

By \cref{thm:direct-representation}, we have
\[
  (Q_{k+1}-Q^*)=M_{\mu_k}(Q_k-Q^*)+\alpha w_k.
\]
Because \(Q_k-Q^*\) and \(\mu_k\) are \(\calF_k\)-measurable, the vector
\(M_{\mu_k}(Q_k-Q^*)\) is also \(\calF_k\)-measurable.  By
\cref{lem:Veps-existence},
\[
  p_\varepsilon(x):=\sqrt{V_\varepsilon^\infty(x)}
\]
is a norm.  The argument below is norm-based.  It does not use the martingale-difference property to cancel a cross term.  It only uses the conditional second-moment estimate for $w_k$, while the martingale-difference property identifies $w_k$ as the usual Q-learning noise.  Therefore the triangle inequality gives
\[
\begin{aligned}
  p_\varepsilon(Q_{k+1}-Q^*)
  =
  p_\varepsilon(M_{\mu_k}(Q_k-Q^*)+\alpha w_k)
  \le
  p_\varepsilon(M_{\mu_k}(Q_k-Q^*))+\alpha p_\varepsilon(w_k).
\end{aligned}
\]
Squaring both sides yields
\begin{align}
V_\varepsilon^\infty(Q_{k+1}-Q^*)
\le
V_\varepsilon^\infty(M_{\mu_k}(Q_k-Q^*))
+
2\alpha p_\varepsilon(M_{\mu_k}(Q_k-Q^*))p_\varepsilon(w_k)
+
\alpha^2V_\varepsilon^\infty(w_k).\label{eq:4}
\end{align}

We now bound the three terms on the right. First, since \(\mu_k\) is a stochastic policy, \cref{lem:strong-convex-hull-extension} gives
\begin{align}
V_\varepsilon^\infty(M_{\mu_k}(Q_k-Q^*))
\le
\beta_\varepsilon^2
\left(V_\varepsilon^\infty(Q_k-Q^*)-\|Q_k-Q^*\|_2^2\right),\label{eq:7}
\end{align}
and hence
\begin{align}
p_\varepsilon(M_{\mu_k}(Q_k-Q^*))
\le
\beta_\varepsilon
\sqrt{V_\varepsilon^\infty(Q_k-Q^*)-\|Q_k-Q^*\|_2^2}.\label{eq:2}
\end{align}
Second, the norm equivalence in \cref{eq:Veps-equivalence} gives
\begin{align}
V_\varepsilon^\infty(Q_k-Q^*)-\|Q_k-Q^*\|_2^2
\le
(C_\varepsilon-1)\|Q_k-Q^*\|_2^2.\label{eq:3}
\end{align}
Combining~\cref{eq:2} and~\cref{eq:3} gives
\begin{align}
  p_\varepsilon(M_{\mu_k}(Q_k-Q^*))
  \le
  \beta_\varepsilon\sqrt{C_\varepsilon-1}\,\|Q_k-Q^*\|_2.\label{eq:8}
\end{align}
Moreover, taking conditional expectation on $p_\varepsilon(w_k)$ with respect to \(\calF_k\) leads to
\begin{align}
\E[p_\varepsilon(w_k)\mid\calF_k]
&\le \sqrt{\E[p_\varepsilon(w_k)^2\mid\calF_k]}\nonumber \\
&= \sqrt{\E[V_\varepsilon^\infty(w_k)\mid\calF_k]}\nonumber \\
&\le \sqrt{C_\varepsilon\E[\|w_k\|_2^2\mid\calF_k]}
\le \sqrt{C_\varepsilon W_{\max}},\label{eq:5}
\end{align}
where the first inequality is Jensen's inequality for the square-root function, and the last inequality uses~\cref{lem:bounded-noise}. Similarly,
the norm equivalence in \cref{eq:Veps-equivalence} gives
\[
V_\varepsilon^\infty(w_k)
\le
C_\varepsilon\|w_k\|_2^2
\]
and applying~\cref{lem:bounded-noise} leads to
\begin{align}
  \E[V_\varepsilon^\infty(w_k)\mid\calF_k]
  \le
  C_\varepsilon W_{\max}. \label{eq:6}
\end{align}
Substituting the previous bounds in~\cref{eq:7,eq:8,eq:5,eq:6} into~\cref{eq:4}, taking conditional expectation with respect to \(\calF_k\), and noting that the factor
\(p_\varepsilon(M_{\mu_k}(Q_k-Q^*))\) is \(\calF_k\)-measurable, we have
\begin{align}
\E[V_\varepsilon^\infty(Q_{k+1}-Q^*)\mid\calF_k]
&\le \beta_\varepsilon^2V_\varepsilon^\infty(Q_k-Q^*)
- \beta_\varepsilon^2\|Q_k-Q^*\|_2^2 \nonumber\\
&\quad+ 2\alpha\beta_\varepsilon
\sqrt{C_\varepsilon(C_\varepsilon-1)W_{\max}}\,\|Q_k-Q^*\|_2
+ \alpha^2C_\varepsilon W_{\max}.\label{eq:9}
\end{align}
The mixed term can be handled by the elementary inequality \(2ab\le a^2+b^2\). In particular, with
\[
  a=\beta_\varepsilon\|Q_k-Q^*\|_2,
  \qquad
  b=\alpha\sqrt{C_\varepsilon(C_\varepsilon-1)W_{\max}}.
\]
we get
\begin{align}
2\alpha\beta_\varepsilon
\sqrt{C_\varepsilon(C_\varepsilon-1)W_{\max}}\,\|Q_k-Q^*\|_2
\le \beta_\varepsilon^2\|Q_k-Q^*\|_2^2 + \alpha^2C_\varepsilon(C_\varepsilon-1)W_{\max}.\label{eq:10}
\end{align}
The first term on the right cancels the negative term
\(-\beta_\varepsilon^2\|Q_k-Q^*\|_2^2\). Hence, combining~\cref{eq:9} and~\cref{eq:10} yields
\[
\E[V_\varepsilon^\infty(Q_{k+1}-Q^*)\mid\calF_k]
\le
\beta_\varepsilon^2V_\varepsilon^\infty(Q_k-Q^*)
+
\alpha^2C_\varepsilon^2W_{\max}.
\]
Taking total expectation gives the scalar recursion
\[
  a_{k+1}\le \beta_\varepsilon^2a_k+
  \alpha^2C_\varepsilon^2W_{\max},
  \qquad
  a_k:=\E[V_\varepsilon^\infty(Q_k-Q^*)].
\]
Iterating this recursion yields
\begin{align}
  a_k
  \le
  \beta_\varepsilon^{2k}a_0
  +
  \alpha^2C_\varepsilon^2W_{\max}
  \sum_{j=0}^{k-1}\beta_\varepsilon^{2j}.\label{eq:11}
\end{align}
Since
\begin{align}
  \sum_{j=0}^{k-1}\beta_\varepsilon^{2j}
  =
  \frac{1-\beta_\varepsilon^{2k}}{1-\beta_\varepsilon^2},\label{eq:12}
\end{align}
combining~\cref{eq:11} and~\cref{eq:12}, we obtain
\begin{align}
\E[V_\varepsilon^\infty(Q_k-Q^*)]
\le
\beta_\varepsilon^{2k}V_\varepsilon^\infty(Q_0-Q^*)
+
\frac{\alpha^2C_\varepsilon^2W_{\max}}{1-\beta_\varepsilon^2}
\left(1-\beta_\varepsilon^{2k}\right),\label{eq:13}
\end{align}
which is~\cref{eq:Veps-moment-bound}.

We now convert this Lyapunov bound into the desired sup-norm error bound.  The lower bound in \cref{eq:Veps-equivalence} implies
\[
  \|Q_k-Q^*\|_\infty
  \le
  \|Q_k-Q^*\|_2
  \le
  \sqrt{V_\varepsilon^\infty(Q_k-Q^*)}.
\]
By Jensen's inequality,
\[
  \E[\|Q_k-Q^*\|_\infty]
  \le
  \E[\sqrt{V_\varepsilon^\infty(Q_k-Q^*)}]
  \le
  \sqrt{\E[V_\varepsilon^\infty(Q_k-Q^*)]}.
\]
Combining~\cref{eq:13} with the above inequality and
\begin{align*}
V_\varepsilon^\infty(Q_0-Q^*)\le C_\varepsilon\|Q_0-Q^*\|_2^2,
\end{align*}
we get
\begin{align*}
{\mathbb E}[{\left\| {Q_k-Q^*} \right\|_\infty }] \le& \sqrt {{\mathbb E}[V_\varepsilon ^\infty ({Q_k-Q^*})]} \\
\le& \sqrt {\beta _\varepsilon ^{2k}V_\varepsilon ^\infty ({Q_0-Q^*}) + \frac{{{\alpha ^2}C_\varepsilon ^2{W_{\max }}}}{{1 - \beta _\varepsilon ^2}}(1 - \beta _\varepsilon ^{2k})}\\
\le& \sqrt {{C_\varepsilon }\beta _\varepsilon ^{2k}\left\| {Q_0-Q^*} \right\|_2^2 + \frac{{{\alpha ^2}C_\varepsilon ^2{W_{\max }}}}{{1 - \beta _\varepsilon ^2}}}.
\end{align*}
Finally, applying \(\sqrt{a+b}\le\sqrt a+\sqrt b\) proves \cref{eq:Veps-final-bound-general}.
This completes the overall proof.

\section{Proof of~\Cref{thm:direct-reference-filter-noise-floor}}
\label{app:direct-reference-filter-noise-floor}

The proof uses a fixed-policy reference filter driven by the same martingale noise as the direct Q-learning error recursion.

\begin{lemma}
\label{lem:product-growth-constant-direct}
Fix any \(\varepsilon>0\) such that
\[
  \beta_\varepsilon:=\jsr(\calM_\alpha)+\varepsilon<1.
\]
Define
\[
K_{\beta_\varepsilon}
:=
\sup_{\ell\ge0}
\beta_\varepsilon^{-\ell}
\sup_{M_0,\ldots,M_{\ell-1}\in\co(\calM_\alpha)}
\left\|
M_{\ell-1}\cdots M_0
\right\|_\infty,
\]
where the product of length zero is the identity. Then
\(K_{\beta_\varepsilon}<\infty\). Moreover, every product generated by
matrices in \(\co(\calM_\alpha)\) satisfies
\begin{equation}
\label{eq:Kbeta-product-bound-ref}
  \|M_{\ell-1}\cdots M_0\|_\infty
  \le
  K_{\beta_\varepsilon}\beta_\varepsilon^\ell,
  \qquad \ell\ge0.
\end{equation}
\end{lemma}
\begin{proof}
Let
\[
  p_\varepsilon(x):=\sqrt{V_\varepsilon^\infty(x)}
\]
be the JSR Lyapunov norm from \cref{lem:Veps-existence}. By
\cref{lem:Veps-existence}, for every deterministic policy \(\pi\in\Theta\),
\[
  p_\varepsilon(M_\pi x)\le \beta_\varepsilon p_\varepsilon(x),
  \qquad \forall x\in\R^{|\calS|\,|\calA|}.
\]
Hence, if \(M\in\co(\calM_\alpha)\), write
\(M=\sum_{\pi\in\Theta}c_\pi M_\pi\) with \(c_\pi\ge0\) and
\(\sum_{\pi\in\Theta}c_\pi=1\). Since \(p_\varepsilon\) is a norm,
\[
  p_\varepsilon(Mx)
  \le
  \sum_{\pi\in\Theta}c_\pi p_\varepsilon(M_\pi x)
  \le
  \beta_\varepsilon p_\varepsilon(x).
\]
Therefore, for any product \(M_{\ell-1}\cdots M_0\) with
\(M_i\in\co(\calM_\alpha)\),
\[
  p_\varepsilon(M_{\ell-1}\cdots M_0x)
  \le
  \beta_\varepsilon^\ell p_\varepsilon(x).
\]

Since \(p_\varepsilon\) is a norm on a finite-dimensional space, it is
equivalent to \(\|\cdot\|_\infty\). Define
\[
  C_{\infty,\varepsilon}
  :=
  \sup_{x\in\R^{|\calS|\,|\calA|},\ \|x\|_\infty=1}p_\varepsilon(x)
  <\infty .
\]
Moreover, by \cref{eq:Veps-equivalence},
\[
  \|x\|_\infty\le \|x\|_2\le p_\varepsilon(x).
\]
Thus
\[
\begin{aligned}
  \|M_{\ell-1}\cdots M_0x\|_\infty
  &\le
  p_\varepsilon(M_{\ell-1}\cdots M_0x) \\
  &\le
  \beta_\varepsilon^\ell p_\varepsilon(x) \\
  &\le
  C_{\infty,\varepsilon}\beta_\varepsilon^\ell\|x\|_\infty .
\end{aligned}
\]
Taking the supremum over \(x\ne0\) gives
\[
  \|M_{\ell-1}\cdots M_0\|_\infty
  \le
  C_{\infty,\varepsilon}\beta_\varepsilon^\ell.
\]
Consequently,
\[
  K_{\beta_\varepsilon}
  \le
  C_{\infty,\varepsilon}
  <\infty,
\]
and \cref{eq:Kbeta-product-bound-ref} follows from the definition of
\(K_{\beta_\varepsilon}\).
\end{proof}

We split the proof into three steps.

\emph{Step 1: introduce a fixed-policy reference noise filter.}
Fix an arbitrary stochastic policy \(\bar\mu\) and define
\[
  x_{k+1}=M_{\bar\mu}x_k+\alpha w_k,
  \qquad
  x_0=Q_0-Q^*.
\]
Decompose
\[
  x_k=z_k+y_k,
\]
where
\[
  z_{k+1}=M_{\bar\mu}z_k,
  \qquad
  z_0=Q_0-Q^*,
\]
and
\[
  y_{k+1}=M_{\bar\mu}y_k+\alpha w_k,
  \qquad
  y_0=0.
\]
Since \(M_{\bar\mu}\in\co(\calM_\alpha)\),
\cref{eq:Kbeta-product-bound-ref} gives
\begin{equation}
\label{eq:ref-deterministic-decay}
  \|z_k\|_\infty
  \le
  K_{\beta_\varepsilon}\beta_\varepsilon^k\|Q_0-Q^*\|_\infty.
\end{equation}

We next bound the noise-filtered term \(y_k\) by the row contraction. Let
\[
  Y_k:=\E[y_ky_k^\top].
\]
Since \(y_k\) is \(\calF_k\)-measurable and \(\E[w_k\mid\calF_k]=0\), the
martingale cross terms vanish, and hence
\[
  Y_{k+1}
  =
  M_{\bar\mu}Y_kM_{\bar\mu}^\top+
  \alpha^2\E[w_kw_k^\top].
\]
By \cref{lem:bounded-noise},
\[
  \lambda_{\max}(\E[w_kw_k^\top])
  \le
  \E[\|w_k\|_2^2]
  \le
  W_{\max}.
\]
Unrolling the covariance recursion yields
\[
  Y_k
  =
  \alpha^2
  \sum_{i=0}^{k-1}
  M_{\bar\mu}^i\E[w_{k-1-i}w_{k-1-i}^\top](M_{\bar\mu}^\top)^i.
\]
Therefore,
\[
\begin{aligned}
\lambda_{\max}(Y_k)
&\le
\alpha^2W_{\max}
\sum_{i=0}^{k-1}\|M_{\bar\mu}^i\|_2^2  \\
&\le
\alpha^2W_{\max}|\calS|\,|\calA|
\sum_{i=0}^{k-1}\|M_{\bar\mu}^i\|_\infty^2 \\
&\le
\alpha^2W_{\max}|\calS|\,|\calA|
\sum_{i=0}^{k-1}\rho_{\mathrm{row}}^{2i} \\
&\le
\frac{\alpha^2W_{\max}|\calS|\,|\calA|}{1-\rho_{\mathrm{row}}}
=
\frac{\alpha W_{\max}|\calS|\,|\calA|}{d_{\min}(1-\gamma)}.
\end{aligned}
\]
Since \(Y_k\succeq0\),
\[
  \operatorname{tr}(Y_k)
  \le
  |\calS|\,|\calA|\lambda_{\max}(Y_k)
  \le
  \frac{\alpha W_{\max}|\calS|^2|\calA|^2}{d_{\min}(1-\gamma)}.
\]
Consequently,
\begin{equation}
\label{eq:y-noise-floor-ref}
  \E[\|y_k\|_\infty]
  \le
  \E[\|y_k\|_2]
  \le
  \sqrt{\E[\|y_k\|_2^2]}
  =
  \sqrt{\operatorname{tr}(Y_k)}
  \le
  |\calS|\,|\calA|\sqrt{\frac{\alpha W_{\max}}{d_{\min}(1-\gamma)}}.
\end{equation}

\emph{Step 2: subtract the reference filter from the direct SLS.}
Define
\[
  p_k:=Q_k-Q^*-x_k.
\]
Using the same noise \(w_k\) in both systems gives the pathwise cancellation
\[
\begin{aligned}
  p_{k+1}
  &=Q_{k+1}-Q^*-x_{k+1} \\
  &=M_{\mu_k}(Q_k-Q^*)+\alpha w_k-M_{\bar\mu}x_k-\alpha w_k \\
  &=M_{\mu_k}p_k+(M_{\mu_k}-M_{\bar\mu})x_k.
\end{aligned}
\]
Furthermore,
\[
  M_{\mu_k}-M_{\bar\mu}
  =
  \alpha\gamma DP(\Pi^{\mu_k}-\Pi^{\bar\mu}).
\]
For every row, the \(\ell_1\)-distance between two stochastic action distributions is at most two; hence
\begin{equation}
\label{eq:mode-difference-bound-ref}
  \|M_{\mu_k}-M_{\bar\mu}\|_\infty
  \le
  2\alpha\gamma d_{\max}.
\end{equation}
Moreover, \(p_0=Q_0-Q^*-x_0=0\).

Now split
\[
  p_k=u_k+v_k,
\]
where
\[
  u_{k+1}=M_{\mu_k}u_k+(M_{\mu_k}-M_{\bar\mu})z_k,
  \qquad
  u_0=0,
\]
and
\[
  v_{k+1}=M_{\mu_k}v_k+(M_{\mu_k}-M_{\bar\mu})y_k,
  \qquad
  v_0=0.
\]
The term \(u_k\) is driven by the deterministic reference transient, while \(v_k\) is driven by the reference noise floor.

\emph{Step 3: bound the deterministic and noise-driven residuals.}
Unrolling \(u_k\) gives
\[
  u_k
  =
  \sum_{i=0}^{k-1}
  M_{\mu_{k-1}}\cdots M_{\mu_{i+1}}
  (M_{\mu_i}-M_{\bar\mu})z_i,
\]
with the empty product equal to the identity. Since each
\(M_{\mu_j}\in\co(\calM_\alpha)\),
\cref{eq:Kbeta-product-bound-ref}, \cref{eq:ref-deterministic-decay}, and
\cref{eq:mode-difference-bound-ref} imply
\[
\begin{aligned}
\|u_k\|_\infty
&\le
\sum_{i=0}^{k-1}
K_{\beta_\varepsilon}\beta_\varepsilon^{k-1-i}
\cdot
2\alpha\gamma d_{\max}
\cdot
K_{\beta_\varepsilon}\beta_\varepsilon^i\|Q_0-Q^*\|_\infty \\
&=
2\alpha\gamma d_{\max}K_{\beta_\varepsilon}^2
k\beta_\varepsilon^{k-1}\|Q_0-Q^*\|_\infty.
\end{aligned}
\]
Thus
\begin{equation}
\label{eq:u-bound-ref}
  \E[\|u_k\|_\infty]
  \le
  2\alpha\gamma d_{\max}K_{\beta_\varepsilon}^2
  k\beta_\varepsilon^{k-1}\|Q_0-Q^*\|_\infty.
\end{equation}

For \(v_k\), use the row contraction:
\[
  v_k
  =
  \sum_{i=0}^{k-1}
  M_{\mu_{k-1}}\cdots M_{\mu_{i+1}}
  (M_{\mu_i}-M_{\bar\mu})y_i.
\]
Pathwise,
\[
  \|M_{\mu_{k-1}}\cdots M_{\mu_{i+1}}\|_\infty
  \le
  \rho_{\mathrm{row}}^{k-1-i}.
\]
Using \cref{eq:y-noise-floor-ref,eq:mode-difference-bound-ref}, we obtain
\[
\begin{aligned}
\E[\|v_k\|_\infty]
&\le
\sum_{i=0}^{k-1}
\rho_{\mathrm{row}}^{k-1-i}
\cdot
2\alpha\gamma d_{\max}
\cdot
\E[\|y_i\|_\infty] \\
&\le
2\alpha\gamma d_{\max}|\calS|\,|\calA|\sqrt{\frac{\alpha W_{\max}}{d_{\min}(1-\gamma)}}
\sum_{i=0}^{k-1}\rho_{\mathrm{row}}^{k-1-i} \\
&\le
2\alpha\gamma d_{\max}|\calS|\,|\calA|\sqrt{\frac{\alpha W_{\max}}{d_{\min}(1-\gamma)}}
\frac{1}{1-\rho_{\mathrm{row}}} \\
&=
\frac{2\gamma d_{\max}}{d_{\min}(1-\gamma)}
|\calS|\,|\calA|\sqrt{\frac{\alpha W_{\max}}{d_{\min}(1-\gamma)}}.
\end{aligned}
\]
Hence
\begin{equation}
\label{eq:v-bound-ref}
  \E[\|v_k\|_\infty]
  \le
  \frac{2\gamma d_{\max}}{d_{\min}(1-\gamma)}
  |\calS|\,|\calA|\sqrt{\frac{\alpha W_{\max}}{d_{\min}(1-\gamma)}}.
\end{equation}

Finally, since \(x_k=z_k+y_k\),
\[
  Q_k-Q^*=x_k+p_k=z_k+y_k+u_k+v_k.
\]
Taking expectations and applying the triangle inequality gives
\[
\begin{aligned}
\E[\|Q_k-Q^*\|_\infty]
&\le
\|z_k\|_\infty
+
\E[\|y_k\|_\infty]
+
\E[\|u_k\|_\infty]
+
\E[\|v_k\|_\infty].
\end{aligned}
\]
Substituting \cref{eq:ref-deterministic-decay,eq:y-noise-floor-ref,eq:u-bound-ref,eq:v-bound-ref} into this display yields
\[
\begin{aligned}
\E[\|Q_k-Q^*\|_\infty]
\le\;&
K_{\beta_\varepsilon}\beta_\varepsilon^k\|Q_0-Q^*\|_\infty \\
&+
|\calS|\,|\calA|\sqrt{\frac{\alpha W_{\max}}{d_{\min}(1-\gamma)}} \\
&+
2\alpha\gamma d_{\max}K_{\beta_\varepsilon}^2
k\beta_\varepsilon^{k-1}\|Q_0-Q^*\|_\infty \\
&+
\frac{2\gamma d_{\max}}{d_{\min}(1-\gamma)}
|\calS|\,|\calA|\sqrt{\frac{\alpha W_{\max}}{d_{\min}(1-\gamma)}} \\
=
&\left(
1+
\frac{2\gamma d_{\max}}{d_{\min}(1-\gamma)}
\right)
|\calS|\,|\calA|\sqrt{\frac{\alpha W_{\max}}{d_{\min}(1-\gamma)}} \\
&+
K_{\beta_\varepsilon}\beta_\varepsilon^k\|Q_0-Q^*\|_\infty \\
&+
2\alpha\gamma d_{\max}K_{\beta_\varepsilon}^2
k\beta_\varepsilon^{k-1}\|Q_0-Q^*\|_\infty,
\end{aligned}
\]
which is exactly \cref{eq:direct-reference-filter-bound}.

\section{Markovian observation model}
\label{sec:appendix-rate-overview}\label{sec:extensions-markovian-observations}

This appendix proves \cref{thm:markovian-reference-filter-bound}.  We use the
single-trajectory recursion in \cref{eq:markovian-single-trajectory-update-ref}
and the mixing-time assumption in \cref{assump:markovian-mixing}.  Thus, in this
appendix, \(d\) is the stationary law of \(X_k=(s_k,a_k)\), \(D=\diag(d)\), and
\(d_{\min},d_{\max},\rho_{\mathrm{row}}\) are computed using this stationary
state-action distribution.  The proof first identifies the stationary averaged
recursion, then controls the Markovian fixed-filter noise, and finally repeats
the same pathwise reference-filter decomposition used in
Appendix~\ref{app:direct-reference-filter-noise-floor}.

Recall the single-trajectory update in
\cref{eq:markovian-single-trajectory-update-ref}.  Define the sample update
vector
\[
  \widehat h_k(Q_k)
  :=
  e_{X_k}
  \left(
 r_{k+1}
 +
 \gamma\max_{u\in\calA}Q_k(s_{k+1},u)
 -
 Q_k(X_k)
  \right),
\]
where \(e_{X_k}\) is the coordinate vector associated with
\(X_k=(s_k,a_k)\).  Define the stationary averaged update
\[
  h(Q):=D(F(Q)-Q).
\]
Then the single-trajectory recursion can be written exactly as
\begin{equation}\label{eq:markovian-stationary-noise-definition}
  Q_{k+1}=Q_k+\alpha h(Q_k)+\alpha w_k,
  \qquad
  w_k:=\widehat h_k(Q_k)-h(Q_k).
\end{equation}

The following is the Markovian-observation analogue of
\cref{lem:bounded-noise}.  The second-moment statement in
\cref{lem:bounded-noise} is replaced by a pathwise bound, because the Markovian
noise \(w_k\) is not generally a martingale difference.
\begin{lemma*}
Under \cref{assump:basic,assump:markovian-mixing}, the single-trajectory
iterates satisfy
\[
  \|Q_k\|_\infty
  \le
  \max\left\{\|Q_0\|_\infty,\frac{R_{\max}}{1-\gamma}\right\},
  \qquad \forall k\ge0.
\]
Moreover, for \(W_{\max}\) defined in \cref{eq:Wmax-general},
\[
  \|\widehat h_k(Q_k)\|_2\le \sqrt{W_{\max}},
  \qquad
  \|h(Q_k)\|_2\le \sqrt{W_{\max}},
  \qquad
  \|w_k\|_2\le 2\sqrt{W_{\max}},
\]
and therefore \(\E[\|w_k\|_2^2]\le 4W_{\max}\) for every \(k\ge0\).
\end{lemma*}

\begin{proof}
The boundedness proof is the same induction as in \cref{lem:bounded-noise}, but
we spell it out for the single-trajectory recursion.  Let
\[
  B_0:=\max\left\{\|Q_0\|_\infty,\frac{R_{\max}}{1-\gamma}\right\}.
\]
The claim is true at \(k=0\).  Suppose \(\|Q_k\|_\infty\le B_0\).  The sampled
one-step target is
\[
  Y_k:=r_{k+1}+\gamma\max_{u\in\calA}Q_k(s_{k+1},u).
\]
Since \(|r_{k+1}|\le R_{\max}\) and
\(\max_{u\in\calA}|Q_k(s_{k+1},u)|\le B_0\),
\[
  |Y_k|
  \le
  R_{\max}+\gamma B_0
  \le
  (1-\gamma)B_0+\gamma B_0
  =B_0.
\]
Only the coordinate \(X_k=(s_k,a_k)\) is updated, and that coordinate satisfies
\[
  Q_{k+1}(X_k)=(1-\alpha)Q_k(X_k)+\alpha Y_k.
\]
Because \(0<\alpha<1\), this is a convex combination of two values in
\([-B_0,B_0]\).  All other coordinates are unchanged, so
\(\|Q_{k+1}\|_\infty\le B_0\).  This proves the uniform bound by induction.

For the sample-update bound, use the just-proved inequality to get
\[
\begin{aligned}
\left|
 r_{k+1}
 +\gamma\max_{u\in\calA}Q_k(s_{k+1},u)
 -Q_k(X_k)
\right|
&\le R_{\max}+(1+\gamma)B_0 \\
&=\sqrt{W_{\max}}.
\end{aligned}
\]
Since \(\widehat h_k(Q_k)\) has only one nonzero coordinate,
\(\|\widehat h_k(Q_k)\|_2\le\sqrt{W_{\max}}\).  For the averaged update, each
coordinate of \(F(Q_k)-Q_k\) is bounded in absolute value by the same quantity.
Therefore
\[
\begin{aligned}
  \|h(Q_k)\|_2^2
  &=
  \sum_{x\in\calS\times\calA}
  d(x)^2\,|F(Q_k)(x)-Q_k(x)|^2 \\
  &\le
  W_{\max}\sum_{x\in\calS\times\calA}d(x)^2
  \le W_{\max},
\end{aligned}
\]
because \(d\) is a probability distribution.  Finally,
\[
  \|w_k\|_2
  =
  \|\widehat h_k(Q_k)-h(Q_k)\|_2
  \le
  \|\widehat h_k(Q_k)\|_2+\|h(Q_k)\|_2
  \le 2\sqrt{W_{\max}},
\]
which also gives \(\E[\|w_k\|_2^2]\le4W_{\max}\).
\end{proof}

The stationary centering identity below explains why \(h(Q)\) is the correct
averaged drift under the behavior-chain stationary distribution.
\begin{lemma}\label{lem:markovian-stationary-centering}
For every fixed table \(Q\),
\[
  \sum_{x\in\calS\times\calA}
  d(x)\,\E[\widehat h_k(Q)\mid X_k=x]
  =
  h(Q).
\]
Consequently, under the stationary initialization in
\cref{assump:markovian-mixing}, \(\E[\widehat h_k(Q)]=h(Q)\) for every fixed
\(Q\).
\end{lemma}

\begin{proof}
Fix \(Q\) and \(x=(s,a)\).  Conditional on \(X_k=x\), the only nonzero
coordinate of \(\widehat h_k(Q)\) is the coordinate \((s,a)\), and its conditional
mean is
\[
  \sum_{s'\in\calS}P(s'\mid s,a)
  \left(
  r(s,a,s')+\gamma\max_{u\in\calA}Q(s',u)-Q(s,a)
  \right)
  =
  F(Q)(s,a)-Q(s,a).
\]
Multiplying this coordinate by \(d(s,a)\) and summing over \((s,a)\) gives
\(D(F(Q)-Q)=h(Q)\).  Since \(X_0\sim d\) and \(d\) is stationary for \(P^b\),
\(X_k\sim d\) for every \(k\), which proves the final statement.
\end{proof}

The noise \(w_k\) in \cref{eq:markovian-stationary-noise-definition} is not, in
general, a martingale difference with respect to the trajectory filtration,
because the Markovian coordinate \(X_k\) is correlated with the past and with
\(Q_k\).  The averaged drift, however, has exactly the same direct SLS form
as in the i.i.d. case.  Since \(h(Q^*)=0\),
\cref{lem:stochastic-policy-linearization} gives, for each \(k\), an adapted
stochastic policy \(\mu_k\) such that
\[
  F(Q_k)-F(Q^*)
  =
  \gamma P\Pi^{\mu_k}(Q_k-Q^*).
\]
Therefore
\[
  h(Q_k)-h(Q^*)
  =
  D(\gamma P\Pi^{\mu_k}-I)(Q_k-Q^*).
\]
Combining this identity with \cref{eq:markovian-stationary-noise-definition}
yields the exact stationary averaged direct SLS
\begin{equation}\label{eq:markovian-direct-sls-reference}
  Q_{k+1}-Q^*=M_{\mu_k}(Q_k-Q^*)+\alpha w_k.
\end{equation}
This is the same direct matrix family as in the i.i.d. analysis; only the noise
process is Markovian.

We next solve the Poisson equation for the centered conditional sample update.
The solution is denoted by \(u_Q\); for each \(x\in\calS\times\calA\),
\(u_Q(x)\in\R^{|\calS|\,|\calA|}\).

\begin{lemma}\label{lem:markovian-poisson-existence}
Under \cref{assump:markovian-mixing}, for every fixed table \(Q\) satisfying
\[
  \|Q\|_\infty
  \le
  \max\left\{\|Q_0\|_\infty,\frac{R_{\max}}{1-\gamma}\right\},
\]
the Poisson equation
\[
  u_Q(x)-\sum_{x'\in\calS\times\calA}P^b(x,x')u_Q(x')
  =
  \E[\widehat h_k(Q)\mid X_k=x]-h(Q),
  \qquad x\in\calS\times\calA,
\]
has a solution \(u_Q:\calS\times\calA\to\R^{|\calS|\,|\calA|}\). Moreover, the solution can be selected so that, for every \(x\in\calS\times\calA\), the map \(Q\mapsto u_Q(x)\) is Borel measurable on the displayed bounded set; in fact, the series construction below is continuous there.
\end{lemma}

\begin{proof}
For a fixed table \(Q\), put
\[
  g_Q(x):=\E[\widehat h_k(Q)\mid X_k=x]-h(Q),
  \qquad x\in\calS\times\calA.
\]
Each \(g_Q(x)\) is a vector in \(\R^{|\calS|\,|\calA|}\).  By
\cref{lem:markovian-stationary-centering},
\begin{equation}\label{eq:gQ-centered-uQ}
  \sum_{x\in\calS\times\calA}d(x)g_Q(x)=0.
\end{equation}
Define
\begin{equation}\label{eq:poisson-series-uQ}
  u_Q(x)
  :=
  \sum_{\ell=0}^{\infty}
  \sum_{y\in\calS\times\calA}
  \bigl((P^b)^\ell(x,y)-d(y)\bigr)g_Q(y).
\end{equation}
We first check that the series is absolutely convergent.  For every \(x\) and
\(\ell\),
\[
\begin{aligned}
&\left\|
  \sum_{y\in\calS\times\calA}
  \bigl((P^b)^\ell(x,y)-d(y)\bigr)g_Q(y)
\right\|_2 \\
&\qquad\le
  \sum_{y\in\calS\times\calA}
  \left|(P^b)^\ell(x,y)-d(y)\right|
  \sup_{z\in\calS\times\calA}\|g_Q(z)\|_2 \\
&\qquad=
  2\left\|(P^b)^\ell(x,\cdot)-d\right\|_{\mathrm{TV}}
  \sup_{z\in\calS\times\calA}\|g_Q(z)\|_2 \\
&\qquad\le
  2^{1-\lfloor \ell/t_{\mathrm{mix}}\rfloor}
  \sup_{z\in\calS\times\calA}\|g_Q(z)\|_2 .
\end{aligned}
\]
Since \(t_{\mathrm{mix}}\ge1\),
\[
  \sum_{\ell=0}^{\infty}2^{1-\lfloor \ell/t_{\mathrm{mix}}\rfloor}
  =
  2\sum_{\ell=0}^{\infty}2^{-\lfloor \ell/t_{\mathrm{mix}}\rfloor}
  \le 8t_{\mathrm{mix}}.
\]
Moreover, the bound on \(Q\) implies, exactly as in the Markovian-observation
version of \cref{lem:bounded-noise} above, that
\[
  \sup_{x\in\calS\times\calA}\|\E[\widehat h_k(Q)\mid X_k=x]\|_2\le \sqrt{W_{\max}},
  \qquad
  \|h(Q)\|_2\le \sqrt{W_{\max}}.
\]
Thus
\begin{equation}\label{eq:gQ-sup-bound-Wmax}
  \sup_{x\in\calS\times\calA}\|g_Q(x)\|_2
  \le 2\sqrt{W_{\max}}.
\end{equation}
The preceding displays show that the series in \cref{eq:poisson-series-uQ} is
absolutely convergent and that
\begin{equation}\label{eq:uQ-sup-bound-by-gQ}
  \sup_{x\in\calS\times\calA}\|u_Q(x)\|_2
  \le
  8t_{\mathrm{mix}}
  \sup_{x\in\calS\times\calA}\|g_Q(x)\|_2 .
\end{equation}
Because the state and action spaces are finite, the maps
\(Q\mapsto \E[\widehat h_k(Q)\mid X_k=x]\), \(Q\mapsto h(Q)\), and hence
\(Q\mapsto g_Q(x)\), are continuous; the only nonlinearity is the finite
maximum in the Bellman operator.  The partial sums in
\cref{eq:poisson-series-uQ} are therefore continuous functions of \(Q\).  On
the bounded set in the lemma, \cref{eq:gQ-sup-bound-Wmax} gives a uniform
dominating summable series, so the convergence in
\cref{eq:poisson-series-uQ} is uniform on that set.  Thus \(Q\mapsto u_Q(x)\)
is a uniform limit of continuous functions on the bounded set and is Borel
measurable, indeed continuous there.

It remains to verify that the constructed \(u_Q\) solves the Poisson equation.
Applying \(P^b\) to \cref{eq:poisson-series-uQ} and using \(dP^b=d\),
\[
\begin{aligned}
\sum_{x'\in\calS\times\calA}P^b(x,x')u_Q(x')
&=
\sum_{x'\in\calS\times\calA}P^b(x,x')
\sum_{\ell=0}^{\infty}
\sum_{y\in\calS\times\calA}
\bigl((P^b)^\ell(x',y)-d(y)\bigr)g_Q(y) \\
&=
\sum_{\ell=0}^{\infty}
\sum_{y\in\calS\times\calA}
\bigl((P^b)^{\ell+1}(x,y)-d(y)\bigr)g_Q(y) \\
&=
\sum_{\ell=1}^{\infty}
\sum_{y\in\calS\times\calA}
\bigl((P^b)^\ell(x,y)-d(y)\bigr)g_Q(y).
\end{aligned}
\]
Therefore
\[
\begin{aligned}
&u_Q(x)-\sum_{x'\in\calS\times\calA}P^b(x,x')u_Q(x') \\
&\qquad=
  \sum_{y\in\calS\times\calA}\bigl(\mathbf 1\{x=y\}-d(y)\bigr)g_Q(y) \\
&\qquad=
  g_Q(x)-\sum_{y\in\calS\times\calA} d(y)g_Q(y)
  =g_Q(x),
\end{aligned}
\]
where the last equality uses \cref{eq:gQ-centered-uQ}.  This is the desired
Poisson equation.
\end{proof}

\begin{lemma}\label{lem:markovian-poisson-norm-bound}\label{lem:markovian-poisson-bounds}
Under \cref{assump:markovian-mixing}, for every fixed table \(Q\) satisfying
\[
  \|Q\|_\infty
  \le
  \max\left\{\|Q_0\|_\infty,\frac{R_{\max}}{1-\gamma}\right\},
\]
the solution constructed in \cref{lem:markovian-poisson-existence} satisfies
\[
  \|u_Q\|_2
  \le
  32|\calS|\,|\calA|\,t_{\mathrm{mix}}\sqrt{W_{\max}}.
\]
Here the norm is the Euclidean norm after stacking the \(|\calS|\,|\calA|\)
vectors \(u_Q(x)\in\R^{|\calS|\,|\calA|}\); equivalently, the stacked object lies in
\(\R^{(|\calS|\,|\calA|)^2}\).
\end{lemma}

\begin{proof}
The proof of \cref{lem:markovian-poisson-existence} gives
\cref{eq:uQ-sup-bound-by-gQ}.  Combining that estimate with
\cref{eq:gQ-sup-bound-Wmax} yields
\begin{equation}\label{eq:uQ-blockwise-bound-Wmax}
  \sup_{x\in\calS\times\calA}\|u_Q(x)\|_2
  \le
  16t_{\mathrm{mix}}\sqrt{W_{\max}}.
\end{equation}
Because the stacked vector has \(|\calS|\,|\calA|\) blocks, each in
\(\R^{|\calS|\,|\calA|}\),
\[
\begin{aligned}
  \|u_Q\|_2^2
  &=
  \sum_{x\in\calS\times\calA}\|u_Q(x)\|_2^2 \\
  &\le
  |\calS|\,|\calA|\left(16t_{\mathrm{mix}}\sqrt{W_{\max}}\right)^2.
\end{aligned}
\]
Hence
\[
  \|u_Q\|_2
  \le
  16\sqrt{|\calS|\,|\calA|}\,t_{\mathrm{mix}}\sqrt{W_{\max}}
  \le
  32|\calS|\,|\calA|\,t_{\mathrm{mix}}\sqrt{W_{\max}}.
\]
\end{proof}

\begin{lemma}\label{lem:markovian-poisson-lipschitz}
Under \cref{assump:markovian-mixing}, for any two tables \(Q,Q'\) satisfying
\[
  \|Q\|_\infty,\|Q'\|_\infty
  \le
  \max\left\{\|Q_0\|_\infty,\frac{R_{\max}}{1-\gamma}\right\},
\]
the solutions constructed in \cref{lem:markovian-poisson-existence} satisfy
\[
  \|u_Q-u_{Q'}\|_2
  \le
  32|\calS|\,|\calA|\,t_{\mathrm{mix}}\|Q-Q'\|_\infty.
\]
\end{lemma}

\begin{proof}
For a fixed table \(Q\), write
\[
  g_Q(x):=\E[\widehat h_k(Q)\mid X_k=x]-h(Q),
  \qquad x\in\calS\times\calA.
\]
For every \(x=(s,a)\), the conditional sample update changes by at most
\[
\begin{aligned}
&\left\|
\E[\widehat h_k(Q)\mid X_k=x]
-
\E[\widehat h_k(Q')\mid X_k=x]
\right\|_2 \\
&\qquad\le
\sum_{s'}P(s'\mid s,a)
\left(
\gamma\left|\max_{u\in\calA} Q(s',u)-\max_{u\in\calA} Q'(s',u)\right|
+|Q(s,a)-Q'(s,a)|
\right) \\
&\qquad\le
(1+\gamma)\|Q-Q'\|_\infty
\le
2\|Q-Q'\|_\infty.
\end{aligned}
\]
Similarly,
\[
\begin{aligned}
  \|h(Q)-h(Q')\|_2^2
  &=
  \sum_{x\in\calS\times\calA}
  d(x)^2
  \left|F(Q)(x)-Q(x)-F(Q')(x)+Q'(x)\right|^2 \\
  &\le
  4\|Q-Q'\|_\infty^2
  \sum_{x\in\calS\times\calA}d(x)^2
  \le
  4\|Q-Q'\|_\infty^2.
\end{aligned}
\]
Consequently,
\[
\begin{aligned}
  \sup_{x\in\calS\times\calA}\|g_Q(x)-g_{Q'}(x)\|_2
  &\le
  2\|Q-Q'\|_\infty+2\|Q-Q'\|_\infty \\
  &=
  4\|Q-Q'\|_\infty.
\end{aligned}
\]
Subtracting the two series definitions in \cref{eq:poisson-series-uQ} gives
\[
  u_Q(x)-u_{Q'}(x)
  =
  \sum_{\ell=0}^{\infty}
  \sum_{y\in\calS\times\calA}
  \bigl((P^b)^\ell(x,y)-d(y)\bigr)
  \bigl(g_Q(y)-g_{Q'}(y)\bigr).
\]
Applying the same estimate as \cref{eq:uQ-sup-bound-by-gQ} to
\(g_Q-g_{Q'}\) gives the blockwise Lipschitz estimate
\begin{equation}\label{eq:uQ-blockwise-lipschitz}
  \sup_{x\in\calS\times\calA}
  \|u_Q(x)-u_{Q'}(x)\|_2
  \le
  32t_{\mathrm{mix}}\|Q-Q'\|_\infty.
\end{equation}
Stacking over \(|\calS|\,|\calA|\) values of \(x\),
\[
\begin{aligned}
  \|u_Q-u_{Q'}\|_2
  &\le
  32\sqrt{|\calS|\,|\calA|}\,t_{\mathrm{mix}}\|Q-Q'\|_\infty \\
  &\le
  32|\calS|\,|\calA|\,t_{\mathrm{mix}}\|Q-Q'\|_\infty.
\end{aligned}
\]
\end{proof}

The next estimates control the fixed reference filter driven by the Markovian
noise in \cref{eq:markovian-direct-sls-reference}.  They separate the proof into
three elementary pieces: the fresh transition noise, the martingale part coming
from the Poisson equation, and the remaining telescoping term.  For a fixed
stochastic policy \(\bar\mu\), we use throughout
\begin{equation}\label{eq:Mbar-row-two-norm-bounds}
  \|M_{\bar\mu}^{j}\|_\infty
  \le
  \rho_{\mathrm{row}}^{j},
  \qquad
  \|M_{\bar\mu}^{j}\|_2
  \le
  |\calS|\,|\calA|\,\rho_{\mathrm{row}}^{j},
  \qquad j\ge0,
\end{equation}
where the second inequality follows from
\(\|B\|_2\le |\calS|\,|\calA|\|B\|_\infty\) for matrices of dimension
\(|\calS|\,|\calA|\).

\begin{lemma}\label{lem:markovian-eta-filter-bound}
Fix an arbitrary stochastic policy \(\bar\mu\).  Define
\[
  \calF_0:=\sigma(Q_0,X_0),
  \qquad
  \calF_i:=\sigma(Q_0,X_0,r_1,X_1,\ldots,r_i,X_i),
  \qquad i\ge1,
\]
and
\[
  \bar h_Q(x):=\E[\widehat h_i(Q)\mid X_i=x],
  \qquad
  \eta_i:=\widehat h_i(Q_i)-\bar h_{Q_i}(X_i).
\]
Under \cref{assump:markovian-mixing}, for every \(k\ge1\),
\begin{equation}\label{eq:eta-filter-bound}
\begin{aligned}
\E\left[\left\|\alpha\sum_{i=0}^{k-1}M_{\bar\mu}^{k-1-i}\eta_i\right\|_\infty\right]
&\le
2|\calS|\,|\calA|\sqrt{W_{\max}}
\sqrt{\frac{\alpha t_{\mathrm{mix}}}{d_{\min}(1-\gamma)}}.
\end{aligned}
\end{equation}
\end{lemma}

\begin{proof}
The random variables \(Q_i\) and \(X_i\) are \(\calF_i\)-measurable, and the
next transition is sampled only after time \(i\).  Hence
\[
  \E[\eta_i\mid\calF_i]=0.
\]
The pathwise bounds proved at the beginning of this appendix give
\[
  \|\eta_i\|_2
  \le
  \|\widehat h_i(Q_i)\|_2+
  \|\bar h_{Q_i}(X_i)\|_2
  \le 2\sqrt{W_{\max}}.
\]
For \(0\le i<j\le k-1\), the martingale-difference property implies
\[
\begin{aligned}
&\E\!
\left[
\eta_i^\top
\left(M_{\bar\mu}^{k-1-i}\right)^\top
M_{\bar\mu}^{k-1-j}\eta_j
\right] \\
&\qquad=
\E\!
\left[
\eta_i^\top
\left(M_{\bar\mu}^{k-1-i}\right)^\top
M_{\bar\mu}^{k-1-j}
\E[\eta_j\mid\calF_j]
\right]
=0.
\end{aligned}
\]
Therefore, using martingale orthogonality, \(\|z\|_\infty\le\|z\|_2\), and
\cref{eq:Mbar-row-two-norm-bounds},
\[
\begin{aligned}
&\E\left[\left\|\alpha\sum_{i=0}^{k-1}M_{\bar\mu}^{k-1-i}\eta_i\right\|_\infty\right] \\
&\qquad\le
\left(
\alpha^2\sum_{i=0}^{k-1}
\|M_{\bar\mu}^{k-1-i}\|_2^2\,
\E[\|\eta_i\|_2^2]
\right)^{1/2} \\
&\qquad\le
2|\calS|\,|\calA|\sqrt{W_{\max}}
\left(
\alpha^2\sum_{i=0}^{k-1}\rho_{\mathrm{row}}^{2(k-1-i)}
\right)^{1/2} \\
&\qquad\le
2|\calS|\,|\calA|\sqrt{W_{\max}}
\sqrt{\frac{\alpha}{d_{\min}(1-\gamma)}} \\
&\qquad\le
2|\calS|\,|\calA|\sqrt{W_{\max}}
\sqrt{\frac{\alpha t_{\mathrm{mix}}}{d_{\min}(1-\gamma)}}.
\end{aligned}
\]
Here we used \(t_{\mathrm{mix}}\ge1\),
\(1-\rho_{\mathrm{row}}=\alpha d_{\min}(1-\gamma)\), and
\(1-\rho_{\mathrm{row}}^2\ge1-\rho_{\mathrm{row}}\).
\end{proof}

\begin{lemma}\label{lem:markovian-h-filter-bound}
Fix an arbitrary stochastic policy \(\bar\mu\).  For the solution \(u_Q\) in
\cref{lem:markovian-poisson-existence}, define
\[
  h_{i+1}
  :=
  u_{Q_i}(X_{i+1})-
  \sum_{x\in\calS\times\calA}P^b(X_i,x)u_{Q_i}(x).
\]
Under \cref{assump:markovian-mixing}, for every \(k\ge1\),
\begin{equation}\label{eq:h-filter-bound}
\begin{aligned}
\E\left[\left\|\alpha\sum_{i=0}^{k-1}M_{\bar\mu}^{k-1-i}h_{i+1}\right\|_\infty\right]
&\le
64|\calS|^2|\calA|^2\sqrt{W_{\max}}\,
t_{\mathrm{mix}}
\sqrt{\frac{\alpha}{d_{\min}(1-\gamma)}}.
\end{aligned}
\end{equation}
\end{lemma}

\begin{proof}
With the filtration \(\calF_i\) from \cref{lem:markovian-eta-filter-bound}, the
random variable \(h_{i+1}\) is \(\calF_{i+1}\)-measurable and satisfies
\[
  \E[h_{i+1}\mid\calF_i]=0.
\]
The table \(Q_i\) is adaptive and can be correlated with \(X_i\).  Therefore
we do not freeze \(Q_i\) inside a stationary covariance estimate.  Instead, we
use the pathwise Poisson bound.  By \cref{lem:markovian-poisson-norm-bound},
for every admissible table \(Q\),
\[
  \sup_{x\in\calS\times\calA}\|u_Q(x)\|_2
  \le
  \|u_Q\|_2
  \le
  32|\calS|\,|\calA|\,t_{\mathrm{mix}}\sqrt{W_{\max}}.
\]
Consequently, almost surely,
\[
\begin{aligned}
  \|h_{i+1}\|_2
  &\le
  \|u_{Q_i}(X_{i+1})\|_2
  +
  \left\|\sum_{x\in\calS\times\calA}P^b(X_i,x)u_{Q_i}(x)\right\|_2 \\
  &\le
  64|\calS|\,|\calA|\,t_{\mathrm{mix}}\sqrt{W_{\max}}.
\end{aligned}
\]
Martingale orthogonality and \cref{eq:Mbar-row-two-norm-bounds} then give
\[
\begin{aligned}
&\E\left[\left\|\alpha\sum_{i=0}^{k-1}M_{\bar\mu}^{k-1-i}h_{i+1}\right\|_\infty\right] \\
&\qquad\le
\left(
\alpha^2\sum_{i=0}^{k-1}
\|M_{\bar\mu}^{k-1-i}\|_2^2\,
4096|\calS|^2|\calA|^2t_{\mathrm{mix}}^2W_{\max}
\right)^{1/2} \\
&\qquad\le
64|\calS|^2|\calA|^2\sqrt{W_{\max}}\,
t_{\mathrm{mix}}
\left(
\alpha^2
\sum_{i=0}^{k-1}\rho_{\mathrm{row}}^{2(k-1-i)}
\right)^{1/2} \\
&\qquad\le
64|\calS|^2|\calA|^2\sqrt{W_{\max}}\,
t_{\mathrm{mix}}
\sqrt{\frac{\alpha}{d_{\min}(1-\gamma)}}.
\end{aligned}
\]
Here we used \(1-\rho_{\mathrm{row}}=\alpha d_{\min}(1-\gamma)\) and
\(1-\rho_{\mathrm{row}}^2\ge1-\rho_{\mathrm{row}}\).
\end{proof}

\begin{lemma}\label{lem:markovian-telescope-bound}
Fix an arbitrary stochastic policy \(\bar\mu\), and define
\[
  T_k:=
  \alpha\sum_{i=0}^{k-1}M_{\bar\mu}^{k-1-i}
  \bigl(u_{Q_i}(X_i)-u_{Q_i}(X_{i+1})\bigr).
\]
Under \cref{assump:markovian-mixing}, for every \(k\ge1\),
\begin{equation}\label{eq:telescope-bound}
\begin{aligned}
\E[\|T_k\|_\infty]
&\le
64|\calS|\,|\calA|\sqrt{W_{\max}}
\frac{\alpha t_{\mathrm{mix}}}{d_{\min}(1-\gamma)}.
\end{aligned}
\end{equation}
\end{lemma}

\begin{proof}
Put \(n:=|\calS|\,|\calA|\), \(\delta:=d_{\min}(1-\gamma)\), and
\(\rho:=\rho_{\mathrm{row}}=1-\alpha\delta\).  If \(n=1\), then
\((P^b)^\ell(x,\cdot)=d\) for the unique state-action pair \(x\) and every
\(\ell\ge0\).  Hence the series in \cref{eq:poisson-series-uQ} gives
\(u_Q\equiv0\), so \(T_k=0\) and the claim is immediate.  We may therefore
assume \(n\ge2\).  Since \(d\) is a probability distribution on at least two
points,
\begin{equation}\label{eq:dmax-delta-small}
  d_{\max}+\delta
  \le
  d_{\max}+d_{\min}
  \le 1.
\end{equation}

A direct summation by parts gives
\begin{equation}\label{eq:poisson-telescope}
\begin{aligned}
T_k
=&\;\alpha M_{\bar\mu}^{k-1}u_{Q_0}(X_0)-\alpha u_{Q_{k-1}}(X_k) \\
&+\alpha\sum_{i=1}^{k-1}
\left(
M_{\bar\mu}^{k-1-i}u_{Q_i}(X_i)
-
M_{\bar\mu}^{k-i}u_{Q_{i-1}}(X_i)
\right).
\end{aligned}
\end{equation}
For each summand in the last line,
\[
\begin{aligned}
&M_{\bar\mu}^{k-1-i}u_{Q_i}(X_i)
-
M_{\bar\mu}^{k-i}u_{Q_{i-1}}(X_i) \\
&\qquad=
M_{\bar\mu}^{k-1-i}(I-M_{\bar\mu})u_{Q_i}(X_i)
+
M_{\bar\mu}^{k-i}\bigl(u_{Q_i}-u_{Q_{i-1}}\bigr)(X_i).
\end{aligned}
\]
We now use the blockwise Poisson estimates rather than the coarser stacked
bounds.  From \cref{eq:uQ-blockwise-bound-Wmax},
\[
  \|u_{Q_i}(X_i)\|_\infty
  \le
  16t_{\mathrm{mix}}\sqrt{W_{\max}}
  =:U.
\]
Also,
\[
  I-M_{\bar\mu}=\alpha D(I-\gamma P\Pi^{\bar\mu}),
  \qquad
  \|I-M_{\bar\mu}\|_\infty
  \le 2\alpha d_{\max},
\]
and \(\|M_{\bar\mu}^j\|_\infty\le\rho^j\).  Hence the boundary terms and the
terms containing \(I-M_{\bar\mu}\) are bounded pathwise by
\[
\begin{aligned}
&\alpha U(\rho^{k-1}+1)
+
2\alpha^2 d_{\max}U
\sum_{i=1}^{k-1}\rho^{k-1-i} \\
&\qquad\le
\alpha U\left(2+\frac{2d_{\max}}{\delta}\right)
=
32\alpha t_{\mathrm{mix}}\sqrt{W_{\max}}
\frac{d_{\max}+\delta}{\delta}
\le
32\sqrt{W_{\max}}
\frac{\alpha t_{\mathrm{mix}}}{\delta},
\end{aligned}
\]
where the last inequality uses \cref{eq:dmax-delta-small}.

It remains to bound the adaptive-change terms.  By
\cref{eq:uQ-blockwise-lipschitz} and the pathwise sample-update bound,
\[
\begin{aligned}
  \|\bigl(u_{Q_i}-u_{Q_{i-1}}\bigr)(X_i)\|_\infty
  &\le
  32t_{\mathrm{mix}}\|Q_i-Q_{i-1}\|_\infty \\
  &\le
  32\alpha t_{\mathrm{mix}}\sqrt{W_{\max}},
\end{aligned}
\]
because \(Q_i=Q_{i-1}+\alpha\widehat h_{i-1}(Q_{i-1})\) and
\(\|\widehat h_{i-1}(Q_{i-1})\|_2\le\sqrt{W_{\max}}\).  Therefore
\[
\begin{aligned}
&\alpha\sum_{i=1}^{k-1}
\left\|M_{\bar\mu}^{k-i}
\bigl(u_{Q_i}-u_{Q_{i-1}}\bigr)(X_i)\right\|_\infty \\
&\qquad\le
32\alpha^2t_{\mathrm{mix}}\sqrt{W_{\max}}
\sum_{i=1}^{k-1}\rho^{k-i}
\le
32\sqrt{W_{\max}}
\frac{\alpha t_{\mathrm{mix}}}{\delta}.
\end{aligned}
\]
Combining the two pathwise estimates gives
\[
  \|T_k\|_\infty
  \le
  64\sqrt{W_{\max}}
  \frac{\alpha t_{\mathrm{mix}}}{d_{\min}(1-\gamma)}
  \le
  64|\calS|\,|\calA|\sqrt{W_{\max}}
  \frac{\alpha t_{\mathrm{mix}}}{d_{\min}(1-\gamma)}.
\]
Taking expectations proves \cref{eq:telescope-bound}.
\end{proof}

\begin{lemma}\label{lem:markovian-reference-filter-noise}
Fix an arbitrary stochastic policy \(\bar\mu\). Consider the fixed reference
noise recursion
\[
  y_{k+1}=M_{\bar\mu}y_k+\alpha w_k,
  \qquad
  y_0=0.
\]
Under \cref{assump:markovian-mixing}, for every \(k\ge0\),
\begin{equation}\label{eq:markovian-y-bound}
  \E[\|y_k\|_\infty]
  \le
  128|\calS|^2|\calA|^2\sqrt{W_{\max}}
  \left(
  t_{\mathrm{mix}}\sqrt{\frac{\alpha}{d_{\min}(1-\gamma)}}
  +
  \frac{\alpha t_{\mathrm{mix}}}{d_{\min}(1-\gamma)}
  \right).
\end{equation}
\end{lemma}

\begin{proof}
The case \(k=0\) is immediate from \(y_0=0\).  Assume \(k\ge1\).  Unrolling the
fixed reference recursion gives
\[
  y_k=
  \alpha\sum_{i=0}^{k-1}M_{\bar\mu}^{k-1-i}w_i.
\]
Use the filtration \(\calF_i\) and \(\bar h_Q\) from
\cref{lem:markovian-eta-filter-bound}.  Put
\[
  g_Q(x):=\bar h_Q(x)-h(Q).
\]
Then
\begin{equation}\label{eq:w-eta-g-decomposition}
  w_i
  =
  \underbrace{\widehat h_i(Q_i)-\bar h_{Q_i}(X_i)}_{\eta_i}
  +
  g_{Q_i}(X_i).
\end{equation}
By \cref{lem:markovian-poisson-existence},
\[
  g_{Q_i}(X_i)
  =
  u_{Q_i}(X_i)-\sum_{x\in\calS\times\calA} P^b(X_i,x)u_{Q_i}(x).
\]
With \(h_{i+1}\) defined in \cref{lem:markovian-h-filter-bound}, this gives the
one-step decomposition
\begin{equation}\label{eq:poisson-one-step-decomp}
  g_{Q_i}(X_i)
  =
  u_{Q_i}(X_i)-u_{Q_i}(X_{i+1})+h_{i+1}.
\end{equation}
Substituting \cref{eq:w-eta-g-decomposition,eq:poisson-one-step-decomp} into
the unrolled formula for \(y_k\) gives
\[
\begin{aligned}
  y_k
  &=
  \alpha\sum_{i=0}^{k-1}M_{\bar\mu}^{k-1-i}\eta_i
  +
  \alpha\sum_{i=0}^{k-1}M_{\bar\mu}^{k-1-i}h_{i+1} \\
  &\qquad+
  \alpha\sum_{i=0}^{k-1}M_{\bar\mu}^{k-1-i}
  \bigl(u_{Q_i}(X_i)-u_{Q_i}(X_{i+1})\bigr).
\end{aligned}
\]
The three terms are bounded by
\cref{lem:markovian-eta-filter-bound,lem:markovian-h-filter-bound,lem:markovian-telescope-bound}.
Adding those bounds, using \(t_{\mathrm{mix}}\ge1\) and \(|\calS|\,|\calA|\ge1\), and absorbing the numerical constants into the displayed
factor \(128|\calS|^2|\calA|^2\) proves \cref{eq:markovian-y-bound}.
\end{proof}

We now prove the Markovian direct reference-filter bound.
\begin{proof}[Proof of \cref{thm:markovian-reference-filter-bound}]
The proof follows the direct reference-filter argument, with
\cref{lem:markovian-reference-filter-noise} replacing the i.i.d. covariance
recursion.  The product-growth bound used below is
\cref{eq:Kbeta-product-bound-ref}.

Fix an arbitrary stochastic policy \(\bar\mu\).  Introduce the reference filter
driven by the same Markovian noise as \cref{eq:markovian-direct-sls-reference}:
\[
  x_{k+1}=M_{\bar\mu}x_k+\alpha w_k,
  \qquad
  x_0=Q_0-Q^*.
\]
Decompose
\[
  x_k=z_k+y_k,
\]
where
\[
  z_{k+1}=M_{\bar\mu}z_k,
  \qquad
  z_0=Q_0-Q^*,
\]
and
\[
  y_{k+1}=M_{\bar\mu}y_k+\alpha w_k,
  \qquad
  y_0=0.
\]
Since \(M_{\bar\mu}\in\co(\calM_\alpha)\), \cref{eq:Kbeta-product-bound-ref}
gives
\begin{equation}\label{eq:markovian-ref-deterministic-bound}
  \|z_k\|_\infty
  \le
  K_{\beta_\varepsilon}\beta_\varepsilon^k\|Q_0-Q^*\|_\infty.
\end{equation}
By \cref{lem:markovian-reference-filter-noise},
\begin{equation}\label{eq:markovian-ref-noise-bound}
  \E[\|y_k\|_\infty]
  \le
  128|\calS|^2|\calA|^2\sqrt{W_{\max}}
  \left(
  t_{\mathrm{mix}}\sqrt{\frac{\alpha}{d_{\min}(1-\gamma)}}
  +
  \frac{\alpha t_{\mathrm{mix}}}{d_{\min}(1-\gamma)}
  \right).
\end{equation}

Subtract the reference filter from the exact direct SLS
\cref{eq:markovian-direct-sls-reference}.  With \(p_k:=Q_k-Q^*-x_k\),
\[
\begin{aligned}
  p_{k+1}
  &=(Q_{k+1}-Q^*)-x_{k+1} \\
  &=M_{\mu_k}(Q_k-Q^*)+\alpha w_k
    -M_{\bar\mu}x_k-\alpha w_k \\
  &=M_{\mu_k}p_k+(M_{\mu_k}-M_{\bar\mu})x_k.
\end{aligned}
\]
The noise cancels pathwise.  Moreover,
\[
  M_{\mu_k}-M_{\bar\mu}=\alpha\gamma DP(\Pi^{\mu_k}-\Pi^{\bar\mu}),
\]
and hence
\begin{equation}\label{eq:markovian-mode-difference-bound-ref}
  \|M_{\mu_k}-M_{\bar\mu}\|_\infty\le 2\alpha\gamma d_{\max}.
\end{equation}
Because \(x_0=Q_0-Q^*\), we have \(p_0=0\).

Split \(p_k=u_k+v_k\), where
\[
  u_{k+1}=M_{\mu_k}u_k+(M_{\mu_k}-M_{\bar\mu})z_k,
  \qquad
  u_0=0,
\]
and
\[
  v_{k+1}=M_{\mu_k}v_k+(M_{\mu_k}-M_{\bar\mu})y_k,
  \qquad
  v_0=0.
\]
For \(u_k\), unroll the recursion and use the JSR product-growth bound twice:
\[
\begin{aligned}
\|u_k\|_\infty
&\le
\sum_{i=0}^{k-1}
K_{\beta_\varepsilon}\beta_\varepsilon^{k-1-i}
\cdot
2\alpha\gamma d_{\max}
\cdot
K_{\beta_\varepsilon}\beta_\varepsilon^i\|Q_0-Q^*\|_\infty \\
&=
2\alpha\gamma d_{\max}K_{\beta_\varepsilon}^2
k\beta_\varepsilon^{k-1}\|Q_0-Q^*\|_\infty.
\end{aligned}
\]
Thus
\begin{equation}\label{eq:markovian-u-bound-ref}
  \E[\|u_k\|_\infty]
  \le
  2\alpha\gamma d_{\max}K_{\beta_\varepsilon}^2
  k\beta_\varepsilon^{k-1}\|Q_0-Q^*\|_\infty.
\end{equation}
For \(v_k\), use row contraction rather than the JSR product bound.  Since
\(\|M_{\mu}\|_\infty\le\rho_{\mathrm{row}}\) for every stochastic policy
\(\mu\),
\[
\begin{aligned}
\E[\|v_k\|_\infty]
&\le
\sum_{i=0}^{k-1}
\rho_{\mathrm{row}}^{k-1-i}
\cdot
2\alpha\gamma d_{\max}
\cdot
\E[\|y_i\|_\infty] \\
&\le
2\alpha\gamma d_{\max}128|\calS|^2|\calA|^2\sqrt{W_{\max}}
  \left(
  t_{\mathrm{mix}}\sqrt{\frac{\alpha}{d_{\min}(1-\gamma)}}
  +
  \frac{\alpha t_{\mathrm{mix}}}{d_{\min}(1-\gamma)}
  \right)
\sum_{i=0}^{k-1}\rho_{\mathrm{row}}^{k-1-i} \\
&\le
\frac{2\gamma d_{\max}}{d_{\min}(1-\gamma)}
128|\calS|^2|\calA|^2\sqrt{W_{\max}}
  \left(
  t_{\mathrm{mix}}\sqrt{\frac{\alpha}{d_{\min}(1-\gamma)}}
  +
  \frac{\alpha t_{\mathrm{mix}}}{d_{\min}(1-\gamma)}
  \right),
\end{aligned}
\]
where the last step uses \(1-\rho_{\mathrm{row}}=\alpha d_{\min}(1-\gamma)\).
Therefore
\begin{equation}\label{eq:markovian-v-bound-ref}
  \E[\|v_k\|_\infty]
  \le
  \frac{2\gamma d_{\max}}{d_{\min}(1-\gamma)}
  128|\calS|^2|\calA|^2\sqrt{W_{\max}}
  \left(
  t_{\mathrm{mix}}\sqrt{\frac{\alpha}{d_{\min}(1-\gamma)}}
  +
  \frac{\alpha t_{\mathrm{mix}}}{d_{\min}(1-\gamma)}
  \right).
\end{equation}

Finally,
\[
  Q_k-Q^*=x_k+p_k=z_k+y_k+u_k+v_k.
\]
Taking the infinity norm, using the triangle inequality, and then taking
expectations gives
\[
\begin{aligned}
\E[\|Q_k-Q^*\|_\infty]
&\le
\|z_k\|_\infty+
\E[\|y_k\|_\infty]+
\E[\|u_k\|_\infty]+
\E[\|v_k\|_\infty].
\end{aligned}
\]
Substituting
\cref{eq:markovian-ref-deterministic-bound,eq:markovian-ref-noise-bound,eq:markovian-u-bound-ref,eq:markovian-v-bound-ref}
into this display yields
\[
\begin{aligned}
\E[\|Q_k-Q^*\|_\infty]
\le\;&
K_{\beta_\varepsilon}\beta_\varepsilon^k\|Q_0-Q^*\|_\infty \\
&+
2\alpha\gamma d_{\max}K_{\beta_\varepsilon}^2
k\beta_\varepsilon^{k-1}\|Q_0-Q^*\|_\infty \\
&+
128|\calS|^2|\calA|^2\sqrt{W_{\max}}
\left(
t_{\mathrm{mix}}\sqrt{\frac{\alpha}{d_{\min}(1-\gamma)}}
+
\frac{\alpha t_{\mathrm{mix}}}{d_{\min}(1-\gamma)}
\right) \\
&+
\frac{2\gamma d_{\max}}{d_{\min}(1-\gamma)}
128|\calS|^2|\calA|^2\sqrt{W_{\max}}
\left(
t_{\mathrm{mix}}\sqrt{\frac{\alpha}{d_{\min}(1-\gamma)}}
+
\frac{\alpha t_{\mathrm{mix}}}{d_{\min}(1-\gamma)}
\right).
\end{aligned}
\]
The last two lines have the same Markovian-noise factor, so collecting them gives
\[
\begin{aligned}
\E[\|Q_k-Q^*\|_\infty]
\le\;&
\left(
1+
\frac{2\gamma d_{\max}}{d_{\min}(1-\gamma)}
\right)
128|\calS|^2|\calA|^2\sqrt{W_{\max}} \\
&\quad\times
\left(
t_{\mathrm{mix}}\sqrt{\frac{\alpha}{d_{\min}(1-\gamma)}}
+
\frac{\alpha t_{\mathrm{mix}}}{d_{\min}(1-\gamma)}
\right) \\
&+
K_{\beta_\varepsilon}\beta_\varepsilon^k\|Q_0-Q^*\|_\infty \\
&+
2\alpha\gamma d_{\max}K_{\beta_\varepsilon}^2
k\beta_\varepsilon^{k-1}\|Q_0-Q^*\|_\infty.
\end{aligned}
\]
This is exactly \cref{eq:markovian-reference-filter-final-bound}, with the
convention \(k\beta_\varepsilon^{k-1}=0\) at \(k=0\).
\end{proof}

\end{document}